\documentclass[10pt]{article} 

\usepackage[preprint]{rlj} 

\usepackage{amssymb}            
\usepackage{mathtools}          
\usepackage{mathrsfs}           
\usepackage{graphicx}           
\usepackage{subcaption}         
\usepackage[space]{grffile}     
\usepackage{url}                
\usepackage{lipsum}             

\usepackage{amsmath,amsfonts,bm}

\def\bpi{{\overline{\pi}}}
\def\bPi{{\overline{\Pi}}}
\def\bd{d}
\def\bmu{{\overline{\mu}}}

\def\eqref#1{equation~\ref{#1}}

\def\1{\bm{1}}

\def\eps{{\epsilon}}

\DeclareMathAlphabet{\mathsfit}{\encodingdefault}{\sfdefault}{m}{sl}
\SetMathAlphabet{\mathsfit}{bold}{\encodingdefault}{\sfdefault}{bx}{n}

\def\gA{{\mathcal{A}}}

\def\gD{{\mathcal{D}}}

\def\gF{{\mathcal{Z}}} %
\def\gG{{\mathcal{G}}}

\def\gI{{\mathcal{I}}}

\def\gM{{\mathcal{M}}}

\def\gS{{\mathcal{S}}}
\def\gT{{\mathcal{T}}}

\def\sN{{\mathbb{N}}}

\def\sR{{\mathbb{R}}}

\newcommand{\E}{\mathbb{E}}

\newcommand{\R}{\mathbb{R}}

\DeclareMathOperator*{\argmax}{\textnormal{argmax}}
\DeclareMathOperator*{\argmin}{\textnormal{argmin}}

\usepackage{ragged2e}
\usepackage{subcaption}
\usepackage{url}
\usepackage{algorithm}
\usepackage{algpseudocode}
\usepackage{amsthm}

\usepackage{amsmath}
\usepackage{amssymb}
\usepackage{mathtools}
\usepackage{amsthm}
\usepackage{xspace}
\usepackage{bm}
\usepackage{bbm}
\usepackage{enumitem}
\usepackage{caption}
\usepackage{diagbox}
\usepackage{booktabs, multirow}
\usepackage{spverbatim}
\usepackage[labelformat=simple]{subcaption}
\usepackage{wrapfig}

\theoremstyle{plain}
\newtheorem{theorem}{Theorem}[section]

\newtheorem{lemma}[theorem]{Lemma}

\theoremstyle{definition}

\newtheorem{assumption}[theorem]{Assumption}
\theoremstyle{remark}

\newcommand{\mc}[1]{\ensuremath{\mathcal{#1}}\xspace}

\newcommand{\tn}[1]{\ensuremath{\textnormal{#1}}\xspace}
\newcommand{\ol}[1]{\ensuremath{\overline{#1}}\xspace}

\usepackage{xcolor}

\usepackage{xcolor}
\usepackage[normalem]{ulem} 

\definecolor{changeRed}{RGB}{200, 0, 0} 

\newif\ifshowchanges

\ifshowchanges
    \newcommand{\rev}[1]{\textcolor{changeRed}{#1}}
    \newcommand{\del}[1]{\textcolor{changeRed}{\sout{#1}}}
\else
    \newcommand{\rev}[1]{#1}
    \newcommand{\del}[1]{} 
\fi

\title{Dense and Diverse Goal Coverage in Multi Goal \\Reinforcement Learning}

\setrunningtitle{Dense and Diverse Goal Coverage in Multi Goal RL}

\author{Sagalpreet Singh\textsuperscript{1}, Rishi Saket\textsuperscript{1}, Aravindan Raghuveer\textsuperscript{1}}

\emails{\{sagalpreet, rishisaket, araghuveer\}@google.com}

\affiliations{
$^{1}$\textbf{Google DeepMind}
}

\contribution{
    This paper formalizes the problem of maximizing expected return while ensuring uniform coverage of goal states within a Multi-Goal RL framework, where a state is determined to be a goal state upon visiting it.
    }
    {
    Multi-Goal RL operates on an MDP where a binary reward (+1 for visiting goal states, 0 otherwise) is provided. Unlike Goal-conditioned RL, which requires pre-defined target goals, or GFlowNets, which are often restricted to deterministic Directed Acyclic Graphs (DAGs), our formulation applies to general MDPs without requiring access to a target distribution.
    }

\contribution{
    This paper identifies a fundamental trade-off between return maximization and diverse goal coverage and introduces a novel objective function designed to balance these.
    }
    {
    Existing methods that encourage exploration via entropy regularization (e.g., Soft Actor-Critic) or intrinsic rewards (e.g., pseudo-counts) are often suboptimal for this task because they incentivize the agent to visit all under-explored states, including non-goal states, rather than focusing diversity only on the goal manifold.
    }

\contribution{
    This paper proposes an iterative algorithm based on the Frank-Wolfe framework to learn a policy mixture that optimizes the proposed objective, providing formal performance guarantees for the resulting policy. 
    }
    {
    While Frank-Wolfe has been utilized for marginal state distribution matching, its application to our specific multi-goal objective is novel. Our approach incorporates a non-trivial offline RL sub-routine for modified reward optimization, allowing for efficient data reuse as new policies are added to the mixture.
    }

\contribution{
    This paper demonstrates that the proposed algorithm achieves superior goal coverage while maintaining high returns compared to state-of-the-art baselines across several continuous control MuJoCo tasks with multiple goals. 
    }
    {
    None
    }


\summary{Reinforcement Learning algorithms are primarily focused on learning a policy that maximizes expected return. As a result, the learned policy can exploit one or few reward sources. However, in many natural situations, it is desirable to learn a policy or a mixture thereof that induces a dispersed marginal state distribution over rewarding states, while still maximizing the expected return which is typically tied to visiting a goal state. This aspect remains relatively unexplored. We formalize the problem of maximizing the expected return while uniformly (almost) visiting the goal states as Multi Goal RL.}

\begin{document}

\makeCover  
\maketitle  

\begin{abstract}
Reinforcement Learning algorithms are primarily focused on learning a policy that maximizes expected return. As a result, the learned policy can exploit one or few reward sources. However, in many natural situations, it is desirable to learn a policy that induces a dispersed marginal state distribution over rewarding states, while maximizing the expected return which is typically tied to visiting a goal state. This aspect remains relatively unexplored. Existing techniques based on entropy regularization and intrinsic rewards use stochasticity for encouraging exploration to find an optimal policy which may not necessarily lead to dispersed marginal state distribution over rewarding states. Other RL algorithms which match a target distribution assume the latter to be available apriori. This may be infeasible in large scale systems where enumeration of all states is not possible and a state is determined to be a goal state only upon visiting it. We formalize the problem of frequently and uniformly visiting the goal states as Multi Goal RL, and propose a novel algorithm that learns a high-return policy mixture with marginal state distribution dispersed over the set of goal states. Our algorithm is based on optimizing a custom RL reward which is computed - based on the current policy mixture - at each iteration for a set of sampled trajectories. The latter are used via an offline RL algorithm to update the policy mixture. We prove performance guarantees for our algorithm, showing efficient convergence bounds for optimizing a natural objective which captures the expected return as well as the dispersion of the marginal state distribution over the goal states. We design and perform experiments on a synthetic MDP and MuJoCo environments to evaluate the effectiveness of our algorithm.
\end{abstract}

\section{Introduction}\label{section:introduction}

Reinforcement Learning (RL) has been effectively applied to various domains  like game playing \citep{mnih2013playing, silver2016mastering}, robotic control \citep{gu2017deep, schulman2015high} and aligning pretrained models \citep{10.5555/3600270.3602281}. The goal is to learn a policy maximizing the expected reward. Many real-world problems fall in the multi-goal settings where the primary objective is to reach an outcome from the subset of desirable outcomes. In such a setting, the reward is a signal of the outcome and not how that specific outcome was arrived at, like the correct placement of a tool by a robotic arm \citep{10.5555/3295222.3295258}, or the discovery of a stable molecular structure \citep{ghugare2024searching}, or reinforcement learning on verifiable rewards \citep{DBLP:journals/corr/abs-2411-15124}.
For each of these examples, there exist multiple desirable outcomes (or goal states), often characterized by a sparse positive scalar reward. 

In real world settings, a policy that utilizes only one goal state (or a subset thereof) may discover only a single stable molecular structure (or a subset thereof). Similarly, the navigation policy of a robot is susceptible to failure if a few goal states become unavailable.
Thus the desirable strategy is to learn a policy (or a policy mixture) that maximizes the return by %
 visiting a large number of goal states rather than exploiting a small subset of them.
Further, in large-scale RL environments the goal states (even if finite) cannot be efficiently enumerated. In particular, the set of goal states is not available apriori, and can only be accessed by sampling trajectories from the environment.  

The standard RL objective is defined via reward functions. RL algorithms learn a policy $\pi$ that maximizes the expected return $J(\pi) = \E\left[\sum_{t=0}^{\infty}\gamma^tr_t\mid\pi\right]$ \citep{10.5555/3312046}.
As a result of such formulation, the agent can learn to exploit a single (or a small subset of) reward source. As a result this standard objective is unsuitable for learning a policy which, while maximizing the reward, visits many goal states.

We abstract out the problem of learning a policy mixture that \emph{frequently} visits a \emph{diverse} set of goal states as follows. Consider for ease of exposition, an MDP with finite states $\mc{S}$, a subset $\mc{S}^+ \subseteq \mc{S}$ of goal states, and a state-only dependent reward function $R: \mc{S} \to \R$ given by $R(s) := \mathbb{I}\{s \in \mc{S}^+\}$. With a discount factor of $\gamma \in (0,1)$ for a given policy define the discounted marginal state distribution $d[\pi]$ as $d[\pi](s) := (1-\gamma)\sum_{t=0}^\infty \gamma^t \Pr[s_t = s\mid \pi]$ where $s_0, s_1, \dots$ is the sequence of states visited in a trajectory sampled from $\pi$. Thus, the expected return $J(\pi) = \E\left[\sum_{t=0}^{\infty}\gamma^t R(s_t)\mid\pi\right]$ equals $(1-\gamma)^{-1}\sum_{s\in \mc{S}^+} d[\pi](s)$. Note that these quantities can be generalized to a mixture $\ol{\pi}$ of policies by taking $\bd[\ol{\pi}](s) := \E_{\pi \sim \ol{\pi}}\left[d[\pi](s)\right]$, and $J(\ol{\pi}) := \E_{\pi \sim \ol{\pi}}[J(\pi)]$. A policy mixture is a probability distribution over policy class. Following a policy mixture means sampling a policy from the policy mixture and following it for an entire episode.

While $J(\ol{\pi})$ would be the traditional return maximization objective, we need to also capture a measure of the diversity of goal states visited.
A natural \textbf{objective} is $\gF(\bpi) := (1-\gamma)J(\ol{\pi}) + \gI^{\mc{S}^+}(d[\ol{\pi}]) = \sum_{s \in \mc{S}^+}(d[\ol{\pi}](s) - d[\ol{\pi}](s)^2/2)$ where $\gI^{X}(p) = -(1/2)\sum_{x\in X} p(x)^2$ is a measure of diversity based on Gini criterion (see \citet{breiman2017classification}).
\rev{The specific choice of Gini criterion as a measure of diversity is justified by the concavity and bounded \emph{curvature constant (see $\mc{C}_\mc{Z}$ defined in Lemma \ref{lemma:objective_properties})} of the resultant objective function.}
In other words, the maximization objective is the sum of the traditional reward and the diversity of the marginal distribution on the goal states. It is easy to see that a policy (assuming it exists) which visits only the goal states and whose discounted marginal state distribution is supported uniformly (and only) over $\mc{S^+}$ maximizes $\gF(\bpi)$ with a value of $1 - 1/(2|\mc{S}^+|)$. Further, uniformly increasing the marginal state distribution probabilities on $\mc{S}^+$, as long as their sum is at most $1$, increases the objective (see Section \ref{sec:optobj} for details). 

{\bf Our Contributions.} We provide an iterative algorithm with offline RL as a subroutine to construct a policy mixture with provable performance guarantees. At each iteration, the algorithm samples $N_T$ trajectories of horizon $H$ as per the existing policy mixture, computes rewards for the observed states using the current policy mixture, and uses $N_\tn{FQI}$ iterations of \emph{Fitted-Q Iteration (FQI)} \citep{LVY19, agarwal2019reinforcement} on the sampled data to obtain a new policy which is added to the policy mixture. We take $\mc{F}$ to be the class of  action-value  or $Q$-value predictors over which FQI optimizes, and the obtained policy is the greedy one w.r.t. the computed $Q$-value predictor. Let $\ol{\pi}^K$ be the policy mixture obtained after $K$ iterations of the algorithm for a finite state MDP. Then, $|\gF(\bpi^K) - \tn{argmax}_{\ol{\pi}}\gF(\bpi)|$ is, with high probability \rev{$1-(\delta + \delta_d)$}, bounded by $O(1/K) + O(\gamma^{H} + \gamma^{N_{\tn{FQI}}}) + O(\epsilon_{\tn{approx}}) + O\left(\sqrt{(HN_T)^{-1}\left[\log(KN_\tn{FQI}\tn{dim}_\mc{F}/\delta) + \tn{dim}_\mc{F}\log(HN_T)\right]}\right)\hspace{0.95em} + O\left(\sqrt{(1/N_T)\log(K/\delta_d)}\right)$
where $\epsilon_{\tn{approx}}$ is the \emph{inherent Bellman error (see Assumption \ref{asn:inherent_error})} and ${\tt dim}_{\mc{F}}$ the \emph{pseudo-dimension} of the class of action-value predictors, assuming a lower bounded \emph{concentrability factor (see Assumption \ref{asn:concentrability})}. Section \ref{section:algorithm} contains the algorithm and Theorem \ref{thm:discrete_bound} formally states the associated performance guarantees. We also extend the algorithm to continuous state-action spaces by using continuous counterpart of FQI \citep{antos2007fitted}.

To evaluate the performance of algorithm empirically, we perform 2 sets of experiments, 1) Discrete MDPs to gain clear insights into the benefits of algorithm by comparing the discounted marginal state distribution induced by the learnt policy mixture with those of the existing algorithms, and 2) Experiments on Reacher, Pusher, Ant and Half Cheetah environments from Brax with multiple goals using SAC, \rev{GFlowNets, }Pseudo Counts and SMM as the baseline methods.

\subsection{Prior Works}\label{section:prior_works}
We present a short survey of related works categorized into distinct topics. For each topic, the corresponding problem setting is described and distinguished from the one addressed in this study.

\textbf{Entropy-Regularized Reinforcement Learning.} 
In widely used RL algorithms like Soft Actor-Critic \citep{haarnoja2018soft}, the actor aims to maximize expected return while also maximizing the entropy of learnt policy by introducing a temperature parameter which determines the relative importance of entropy vs. the return thereby controlling the stochasticity of the optimal policy. 
Similarly, \citet{ahmed2019understanding}, \citet{han2021max} and \citet{abdolmaleki2018maximum} propose methods that use maximum entropy objectives to encourage exploration.
While such formulations do encourage exploration, no direct conclusion can be drawn about the diversity in set the of goal states visited by the policy.
Further, recent works like \citet{zhang2025when} have highlighted how maximum entropy can mislead policy optimization by enforcing high entropy in states where one action is vastly superior compared to other actions. %
As demonstrated by examples in Figures \ref{figure:mdp_explore} and \ref{figure:mdp_perfect}, learning a highly stochastic high return policy may not align with our  objective.%

\begin{figure}[t]
    \begin{subfigure}[t]{0.2\linewidth}
        \includegraphics[width=\linewidth]{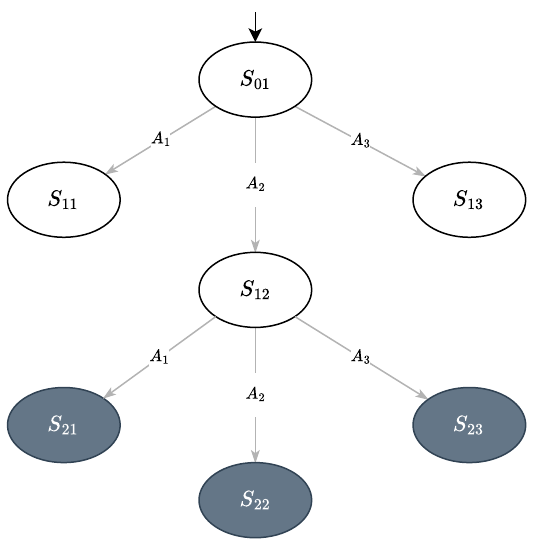}
        \caption{}\label{figure:mdp}
    \end{subfigure}
    \hfill
    \begin{subfigure}[t]{0.2\linewidth}
        \includegraphics[width=\linewidth]{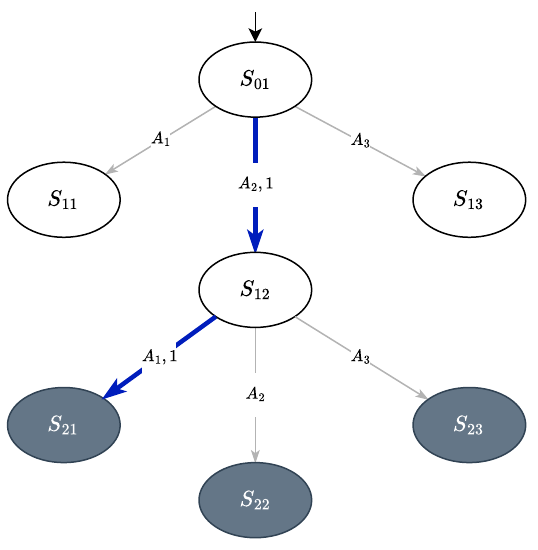}
        \caption{}\label{figure:mdp_exploit}
    \end{subfigure}
    \hfill
    \begin{subfigure}[t]{0.2\linewidth}
        \includegraphics[width=\linewidth]{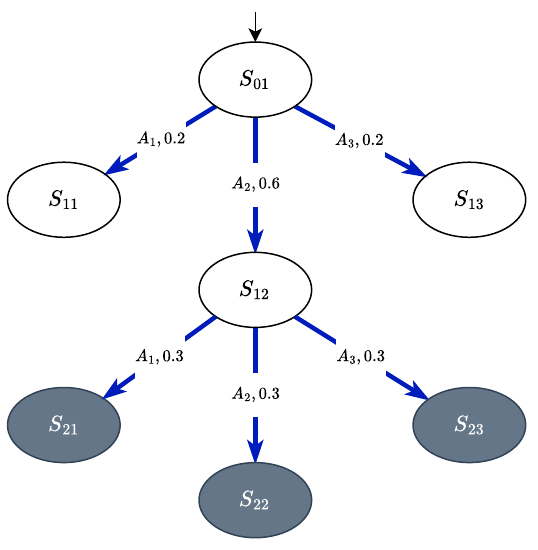}
        \caption{}\label{figure:mdp_explore}
    \end{subfigure}
    \hfill
    \begin{subfigure}[t]{0.2\linewidth}
        \includegraphics[width=\linewidth]{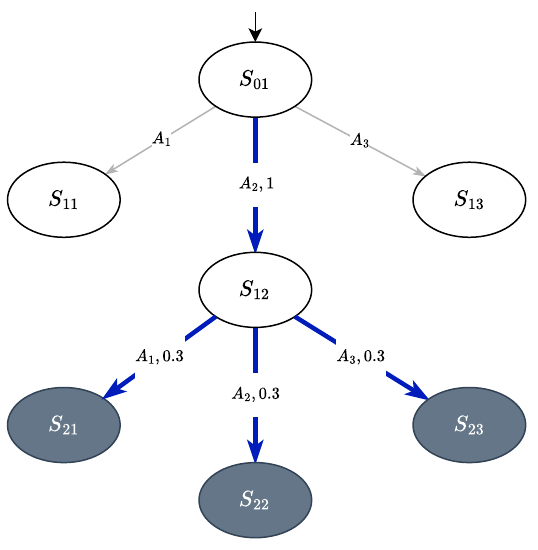}
        \caption{}\label{figure:mdp_perfect}
    \end{subfigure}

    \caption{(a) A continuing deterministic MDP with a start state and 3 goal states (highlighted). Assume that all states with no outgoing edges are sink states i.e. any action in these states leads back to the same state. State $S_{01}$ is critical where enforcing high stochasticity can lead to low returns. However, it is perfectly fine to enforce highest entropy on action selection at state $S_{12}$. (b) An algorithm that encourages exploration may still end up learning a high return policy with low stochasticity which only visits a subset of goal states. (c) A policy that encourages exploration by promoting stochastic policies will enforce high stochasticity at all states (including critical states) which may lead to policies with suboptimal expected returns. (d) A desirable policy is the one which visits a diverse set of goal states without compromising on the expected return.}
\end{figure}

\textbf{Goal Conditioned Reinforcement Learning (GCRL).} 
In GCRL \citep{ijcai2022p770, eysenbach2022contrastive, 10.5555/3045118.3045258, 10.5555/3295222.3295258}, an agent learns to achieve different goals  
by making decisions depending on the goal.
In other words, the agent's actions are conditioned not only on the current observation but also the goal state.
GCRL assumes the knowledge of the specific  goal state during learning as well as inference, which differentiates it from our multi-goal setup.
\rev{Works like Contrastive RL \citep{eysenbach2022contrastive} and Quasimetric RL \citep{myers2025offline} utilize the structure of the state space to learn efficient distance metrics to goals.
Similarly, C-Learning \citep{eysenbach2020c} reformulates goal-reaching as future state density classification.
These methods improve sample efficiency but operate in the paradigm of reaching a queried goal. To address long horizons, HiQL \citep{park2023hiql} and LEAP \citep{nasiriany2019planning} learn high-level policies that propose sub-goals in a latent space for low-level controllers.
While these methods allow reaching distant goals, they rely on the decomposition of a provided task, whereas multi goal RL focuses on the discovery and near-uniform visitation of the set of all rewarding states.}
Some works in GCRL, for e.g. by \citet{zhao2019maximum}, study the problem of maximizing diversity in the set of goal states achieved.
In all these settings, observations are augmented with the goal in an episode for the agent to make decisions to achieve that goal.

\textbf{State Marginal Matching in Reinforcement Learning.} 
\citet{hazan2019provably} study and propose algorithms for learning a policy mixture that optimizes for an objective defined solely as a function of state-visitation frequencies. \citet{lee2019efficient} propose an algorithm for learning a policy mixture for which the marginal state density matches a target density function. These methods are useful in learning exploratory policies or scenarios where the goal state distribution is known apriori. However, absence of access to the target marginal state distribution is the key 
 differentiating factor due to which such techniques are not applicable to our problem setting. 

\textbf{Intrinsic Reward based Reinforcement Learning.} 
Intrinsic rewards, exploration bonus, curiosity enabled exploration, and diversity in visited states and actions taken have been explored quite extensively in existing literature. \citet{chentanez2004intrinsically} proposed learning a hierarchy of reusable skills by providing intrinsic rewards for achieving novel events, rather than for visiting a diverse set of states. \citet{achiam2017surprise} formulated an intrinsic reward based on surprise, defined as the agent's model prediction error, to drive exploration towards regions of the state space  understood poorly by the agent. \citet{pathak2017curiosity} generate an intrinsic reward from the error in predicting the next state's features given the current state and action, which incentivizes the agent to explore novel interactions. \citet{eysenbach2018diversity} learn a set of diverse skills by rewarding policies for visiting states that are different from the states visited by other skills, thereby maximizing state space coverage rather than explicitly learning to visit varied goals. 
The primary focus of these techniques is to encourage exploration in learning a policy without any specific consideration of induced marginal state distribution.

\rev{\textbf{GFlowNets.} The objective of GFlowNets \citep{bengio2021flow, malkin2022trajectory} is to sample structures i.e. visit terminal states with probabilities proportional to a given reward function. It  uses a flow matching objective which requires knowledge of the parents of each state. Further, it also assumes that the transition function is deterministic and does not induce any cycles, as they rely on modeling flow conservation from a source to terminal states without loops. This restricts their direct application in general control environments where cycles are natural (e.g., a robot pacing back and forth). In contrast, multi-goal RL setting is not restricted by such assumptions and is applicable to general MDPs.}

\rev{\textbf{Other Related Works.} Methods like MEGA \citep{pitis2020maximum} and PEG \citep{hu2023planning} prioritize goals with high uncertainty or learning progress to expand the frontier of reachable states focusing on exploration while multi goal RL specifically aims to increase the diversity in visitation over the goal set, rather than just expanding the frontier. Variational Empowerment as Representation Learning \citep{choi2021variational} defines diversity (implicitly via Mutual Information) as skill diversity or state coverage in the context of empowerment. It aims to maximize the agent's ability to reach distinct states, but it does not strictly enforce a uniform probability distribution over a specific subset of correct outcome states defined by an external reward hence distinguishing it from our work.}

\section{Preliminaries}\label{section:preliminaries}
\textbf{MDP.} Consider a continuing (infinite horizon) MDP, $\gM=\{\gS, \gA, R, \gamma, P, \rho_0\}$ where $\gS$ denotes the state space, $\gA$ denotes the action space, $R$ denotes the reward function, $0\leq \gamma < 1$ denotes the discount factor, $P: \gS \times \gA \rightarrow \Delta(\gS)$ denotes the environment dynamics, and $\rho_0 \in \Delta(\gS)$ denotes the start state distribution.

\textbf{State \& Action Spaces.} We assume finite state and action spaces for discrete MDPs with cardinalities $|\gS|,|\gA| \in \sN$ respectively. For continuous MDPs, we assume bounded state and action spaces such that $\gS=[0,1]^{d_\gS}$ and $\gA=[0,1]^{d_\gA}$ where $d_\gS, d_\gA \in \sN$.

\textbf{Reward Function.} As we study the multi goal setting, the reward is a function of only the state. We refer to desirable outcomes/states as \textit{goal} states and the rest as \textit{non-goal} states. Formally, the state space is partitioned into sets of goal ($\gS^+$) and non-goal ($\gS^-$) states such that $\gS=\gS^+ \cup \gS^-$ and $\gS^+ \cap \gS^-=\phi$. The reward function can be defined as $R(s) = \begin{cases} 1, & \text{if } s \in \gS^+ \\ 0, & \text{if } s \in \gS^- \end{cases}$. $R$ is essentially a binary classifier over state space.

\textbf{Policy \& Policy Mixture.} A \emph{policy} $\pi$ function specifies the probability of taking an action from a state.  For any $a\in\gA$ and $s\in\gS$, $\pi(a|s)$ is the probability with which action $a$ is selected in state $s$ when following policy $\pi$. We denote the policy class by $\Pi$. A policy mixture (denoted by $\bpi$) is a probability distribution over the policy class. The class of all valid policy mixtures is denoted by $\bPi=\Delta(\Pi)$. Following a policy mixture means sampling a policy from the mixture and then following it for an entire episode.

{\bf Action-Value predictors.} The action-value or $Q$-value function corresponding to a policy $\pi$ is the expected sum of discounted reward starting from a state and a specific action i.e., $q_{\pi}(s,a) = \E_{\pi}\left[\sum_{t=0}^t \gamma^t R_t\,\mid\, s_0 = s, a_0 = a\right]$. On the other hand, given an action-value function predictor $f : \mc{S} \times \mc{A} \to \R$, the corresponding greedy policy $\pi'$ given by $\pi'(s)\in \tn{argmax}_a f(s, a)$ yields at least as much return (starting from any state) as $\pi$. We take $\mc{F}: \mc{S} \times \mc{A} \to [0, V_\tn{max}]$ to be the class of action-value predictors and use them to derive the greedy policy. Note that $V_\tn{max} = 1/(1-\gamma)$.

To state our results, we shall also use the \emph{pseudo-dimension} of $\mc{F}$, denoted by ${\tt dim}_{\mc{F}}$. For any class of real-valued functions, $\mc{H} : \mc{X} \to \R$, the pseudo-dimension ${\tt dim}_{\mc{H}}$ is defined to be the VC-dimension of the class of binary-valued functions $\{t_h : \mc{X}\times \R \to \R \mid h\in \mc{H}\}$ where $t_h(x,a) = \tn{sign}(h(x) -a)$. 

\textbf{Discounted Marginal State Distribution.} Denote the probability of agent being in state $s$ at timestep $t$ when following a fixed policy $\pi$ for a fixed MDP $\gM(\gS, \gA, R, \gamma, \rho_0)$ as $d_{\pi,\gM}^t(s)$ such that $d_{\pi,\gM}^0(s)=\rho_0$.
Note that $d_{\pi,\gM}^t$ is a probability mass function (PMF) for discrete state spaces and probability density function (PDF) for continuous state spaces. In this exposition, for simplicity we consider only finite state spaces (see paragraph on \textit{Continuous state-action spaces} in Section \ref{section:algorithm}).
Discounted marginal state distribution is defined as $d_{\pi,\gM}(s) := (1-\gamma)\sum_{t=0}^\infty \gamma^t d_{\pi,\gM}^t(s)$.
For brevity, we use $d[\pi](s)$ instead of $d_{\pi,\gM}(s)$.
For a policy mixture $\bpi$, the induced marginal state distribution is denoted by $\bd[\bpi]$. Depending upon whether the policy class is continuous or discrete, $\bd[\bpi]$ is defined as either $\bd[\bpi](s) = \int_{\Pi}d[\pi](s)\bpi(\pi)d\pi$ or $\bd[\bpi](s) = \sum_{\pi \in \Pi}d[\pi](s)\bpi(\pi)$.
Inspired from Gini criterion \citep{breiman2017classification}, we define the diversity of discounted marginal state distribution over set of goal states as $\gI^{\gS^+}(\bpi) = -1/2 \sum_{s\in\gS^+}\bd[\bpi](s)^2$.
In this paper, we use terms discounted marginal state distribution and marginal state distribution interchangeably to mean the former unless explicitly mentioned.

\textbf{Expected Return.} Expected return is the discounted sum of rewards received by the agent when following policy $\pi$ for a fixed MDP $\gM(\gS, \gA, R, \gamma, \rho_0)$. Denote the expected return by $J^{\pi}_{\gM}$. For brevity, we use $J(\pi)$ instead of $J^{\pi}_{\gM}$. For a policy mixture $\bpi$, the expected return is denoted by $J(\bpi)$ and defined as $J(\bpi) = \E_{\pi\sim\bpi}\E_{(s_t,a_t,r_t,s_{t+1})\sim\pi}[\sum_{i=0}^\infty \gamma^{t}r_{t}] = 1/(1-\gamma) \sum_{s\in \gS} R(s)\bd[\bpi](s) = 1/(1-\gamma) \sum_{s\in \gS^+} \bd[\bpi](s)$. We also define $J_\gamma(\bpi) = (1-\gamma)J(\bpi) = \sum_{s\in \gS^+} \bd[\bpi](s)$.

\subsection{Optimization Objective}\label{sec:optobj} 
The objective is to learn a policy mixture that maximizes return and induces a marginal state distribution that is well dispersed over goal states (i.e. not concentrated on a few goal states).
Due to a discounting factor less than 1 or environment dynamics, maximizing return and inducing well dispersed marginal state distribution over goal states can turn out to be conflicting objectives as demonstrated via examples in Supplementary \ref{appendix:balancing_conflict}.

It is worth noting that the objectives optimized by entropic regularization techniques could benefit by increasing policy entropy even if it ends up assigning more weight to non-goal states. To avoid such undesirable outcomes, we adopt a balanced approach and 
 propose a maximization objective $\max_{\bpi\in\bPi} \gF(\bpi)$, where $\gF(\bpi) = J_\gamma(\bpi) + \gI^{\gS^+}(\bpi)$ i.e. the sum of return and diversity of induced marginal state distribution restricted to goal states. Thus, the objective is:
\begin{equation}\label{eqn:objective}
    \gF(\bpi) = \sum\limits_{s\in\gS^+} \left(\bd[\bpi](s) - \frac{1}{2}\bd[\bpi](s)^2\right)
\end{equation}
\rev{The specific choice of Gini criterion as a measure of diversity is justified by the concavity and bounded \emph{curvature constant (see $\mc{C}_\mc{Z}$ defined in Lemma \ref{lemma:objective_properties})} of the resultant objective function.} The above optimization objective captures both our goals: maximizing the return and dispersion of the state distribution over $\gS^+$. To see this, let us define $\mu := \sum_{s \in \gS^+}\bd[\bpi](s)/|\gS^+|$ as the average state probability, and $\eps(s) := \bd[\bpi](s) - \mu$, as the deviation from the average. Then,
\begin{eqnarray}
    \gF(\bpi) = |\gS^+|\mu - (1/2)\sum_{s\in \gS^+}\left(\eps(s) + \mu\right)^2 \quad\quad\quad& \nonumber \\
    = |\gS^+|\mu - (1/2)|\gS^+|\mu^2 - (1/2)\sum_{s \in \gS^+}\eps(s)^2&
\end{eqnarray}
where the second equality follows from $\sum_{s\in \gS^+}\eps(s) = 0$.
Therefore, for a fixed $\mu$, the objective increases when the vector of deviations becomes smaller in its Euclidean norm, and is maximized when the deviations are all zero i.e., for all $s \in \gS^+$, $\bd[\bpi](s) = \mu$.

Further, define for a fixed set of probabilities, $Z(c) := \sum\limits_{s\in\gS^+} \left(c\bd[\bpi](s) - \frac{1}{2}c^2\bd[\bpi](s)^2\right)$. Then, $Z'(c) > 0$ for some $c > 1$, when $\sum\limits_{s\in\gS^+}\bd[\bpi](s)  > \sum\limits_{s\in\gS^+}\bd[\bpi](s)^2$. This is always true when  $\sum_{s\in \gS^+}\bd[\bpi](s) \in (0,1)$ as $\bd[\bpi](s) < \bd[\bpi](s)^2$ for each non-zero $\bd[\bpi](s)$. Thus, uniformly increasing the probabilities in $\gS^+$ by some $c > 1$ (as long as the total probability is at most 1)   increases the objective.

We denote the optimal policy mixture for this objective by 
$\bpi^* \in \tn{argmax}_{\bpi\in\bPi} \gF(\bpi)$

Next, we outline key properties of the policy mixture class as well as the objective function. These properties are instrumental in proving convergence guarantees for the proposed algorithm.
\vspace{5pt}
\begin{lemma}\label{lemma:objective_properties}
    (proved in Supplementary \ref{appendix:objective_properties}) The following properties hold for $\bPi$ and $\gF$:
    \begin{enumerate}
        \item $\bPi$ is a convex set.
        \item $\gF$ is a concave function over $\bPi$.
        \item $C_\gF = \sup\limits_{\substack{\bpi_1,\bpi\in\bPi\\\lambda\in(0,1)\\\bpi_2=\bpi_1 +\lambda(\bpi-\bpi_1)}} \frac{2}{\lambda^2}\biggl[\gF(\bpi_1) + \langle\bpi_2-\bpi_1, \nabla_{\pi}\gF(\bpi_1)\rangle -\gF(\bpi_2)\biggr] < \infty$
    \end{enumerate}
\end{lemma}

\section{Algorithm}\label{section:algorithm}

\begin{algorithm*}[ht]
\caption{Dense \& Diverse Goal Coverage (DDGC)}\label{alg:ddgc}
\textbf{Input:} $\gM$ (MDP), $h_d$ (State space discretization function, identity for discrete state spaces)\\
\textbf{Initialize:} $\bpi_0$ (policy mixture)\\
\textbf{Hyper-parameters:} $N_T$ (Number of Trajectories), $H$ (Horizon), $K$ (Mixture Size)
\begin{algorithmic}[1]
\setlength{\itemsep}{5pt}
\State $\Gamma_e \sim \textsc{RND}(\gM)$ \Comment{Sample data using an exploratory strategy}
\For{k in 1 $\cdots$ K}
    \State $\Gamma_k \sim \gM(\bpi_{k-1})$, $|\Gamma_k| = HN_T$ \Comment{\parbox[t]{4.5cm}{Sample data by following $\bpi$,\\$\Gamma_k^{(i)} \in \gS\times\gA\times\{0,1\}\times\gS\times[H]$,\\$[H]:=\{1,2,\cdots, H\}$}}
    \vspace{1em}
    \State $\hat{d}_k(s) = \frac{(1-\gamma)}{N_T(1-\gamma^H)}\sum\limits_{(s,a,r,s',t)\in \Gamma_k}\gamma^{t-1}\mathbb{I}(h_d(s')=h_d(s))$ \Comment{\parbox[t]{4cm}{Normalized discounted state visitation frequency from $\Gamma_k$}}
    \State $\hat{r}_k(s) = \begin{cases} 1-
    \hat{d}_k(s), & s\in\gS^+\\ 0, & s\in\gS^- \end{cases}$ \Comment{Custom reward}
    \State $\gD_k \gets \Gamma_e \cup \left(\bigcup_{i=1}^{k} \Gamma_i\right)$ \Comment{Combine all data for batch RL}
    \State $\gD'_k \gets \{(s,a,\hat{r}_k(s'),s')|(s,a,r,s',t) \in \gD_k\}$ \Comment{Update batch with custom reward}
    \State $\mu_k \gets \text{RL}(\gM, \gD'_k)$ \Comment{Batch RL}
    \State $\bmu_k(\pi) = \begin{cases} 1, & \pi = \mu_k\\ 0, & \pi \neq \mu_k \end{cases}$ \Comment{Policy mixture with all weight on $\mu_k$}
    \State $\bpi_k \gets (1-\lambda_k)\bpi_{k-1} + (\lambda_k)\bmu_k$, $\lambda_k = \frac{2}{k+1}$ \Comment{Mixture update}
\EndFor
\State \Return $\bpi_K$
\end{algorithmic}
\end{algorithm*}

\begin{algorithm*}[ht]
\caption{Fitted Actor Critic}\label{alg:FAC}
\textbf{Input:} $\gM$ (MDP), $\{s,a,r,s'\} \in\Gamma$ (Batch Data)\\
\textbf{Initialize:} $\pi_0$ (policy), $f_0$ (Q Value Approximator)\\
\textbf{Hyper-parameters:} $N_{FQI}$ (Number of Steps)
\begin{algorithmic}[1]
\vspace{0.5em}
\For{k in 1 $\cdots$ $N_{FQI}$}
\vspace{0.5em}
    \State $f_k \gets \argmin_{f\in\mc{F}} \sum\limits_{(s,a,r,s')\in\Gamma} \left(f(s,a)-[r+\gamma f_{k-1}(s',\pi(s'))]\right)^2$ \Comment{Updating Q Value Approximator}
    \vspace{0.5em}
    \State $\pi_k \gets \argmax_{\pi\in\Pi} \sum\limits_{(s,a,r,s')\in\Gamma} f_k(s,\pi(s))$ \Comment{Updating Policy}
\vspace{0.5em}
\EndFor
\State \Return $\bpi_K$
\end{algorithmic}
\end{algorithm*}

To optimize the objective defined in Equation \ref{eqn:objective}, we use Algorithm \ref{alg:ddgc} which is based on Frank–Wolfe algorithm \citep{pmlr-v28-jaggi13}. A policy mixture is randomly initialized and updated at each iteration.

The algorithm iteratively updates the policy mixture by adding a new policy and decaying the weights of existing policies in the mixture. The new policy is obtained by optimizing for a reward function that is dependent on the existing policy mixture. The core idea is to encourage visitation of less frequently visited goal states over more frequently visited goal states as per the existing policy mixture. This is achieved by assigning a reward on goal states which decreases monotonically with state visitation frequency but is always non-negative. It is also important to note that 0 reward is given for visiting non-goal states. Thus, reward on any goal state is always greater than any other non-goal state. This is a clear distinction from algorithms that provide exploration bonus or bonus aimed at encouraging stochasticity in the policies. Note that the reward function is fixed in each iteration and is derived from the estimate of goal state visitation frequency of the policy mixture obtained from the previous iteration. The particular choice of batch RL is useful as it requires state visitation frequency estimation for only the goal states encountered in the trajectories sampled from the policy mixture and not for all the states.

\rev{Further, the algorithm bootstraps with an exploration step to discover goal states and allow for a reasonable value of concentrability constant $\mathcal{C}$ (see Assumption \ref{asn:concentrability}). While the algorithm is agnostic to the choice of exploration strategy, we use RND - Random Network Distillation \citep{burda2018exploration} for efficient exploration.}

We analyze the algorithm and provide performance guarantees on policy mixture obtained at $K$th iteration. The following lemma bounds the sub-optimality gap $h(\bpi_K) = \gF(\bpi^*)-\gF(\bpi_k)$.
\vspace{5pt}
\begin{lemma}\label{lemma:wf_convergence}
    (proved in Supplementary \ref{appendix:wf_convergence}) If performance of the  policy $\bmu_k$ output by the RL algorithm using empirical estimate $\hat{d}$ in $k^{th}$ iteration of Algorithm \ref{alg:ddgc} is $\epsilon_k$ sub-optimal with probability $\geq 1-\delta_k$ over the choice of \ $\Gamma_k$, then $h(\bpi_K) \leq \frac{2C_\gF}{K+1} + \frac{2}{K(K+1)}\sum\limits_{k=1}^K k\epsilon_k$ holds with probability $\geq 1-\sum\limits_{k=1}^K \delta_k$.
\end{lemma}

The proof of Lemma \ref{lemma:wf_convergence} relies on the properties of Frank-Wolfe algorithm. Exploiting the smoothness of our objective $\mathcal{Z}$ (established in Lemma \ref{lemma:objective_properties}), we lower-bound the improvement at each iteration $k$ using the curvature constant $C_{\mathcal{Z}}$. The gradient of $\mathcal{Z}$ at the current mixture defines a linear reward function, casting the direction-finding step as a standard RL problem. Assuming the RL solution to be PAC, the suboptimality gap of a policy obtained after $K$ iterations of Frank-Wolfe algorithm is upper bounded using the union bound argument.

We choose FQI for Batch RL step in the proposed algorithm. Under assumptions \ref{asn:concentrability} and \ref{asn:inherent_error}, we bound the performance of policy which is obtained using FQI to optimize on estimated reward in lemma \ref{lemma:fqi_performance_difference}. These are standard assumptions used for FQI analysis as in \citet{agarwal2019reinforcement} and \citet{LVY19}.
\vspace{5pt}
\begin{assumption}\label{asn:concentrability}
    (Concentrability) There exists a constant $C$ such that for all policy mixtures $\bpi \in \bPi$ and for policy mixture $\bpi_k$ at every step $k$, we have for each $s\in\gS$ and $a\in\gA$
    \begin{equation*}
        \frac{\bd[\bpi(s)]\cdot p_\bpi(a|s)}{\bd[\bpi_k(s)]\cdot p_{\bpi_k}(a|s)} \leq C
    \end{equation*}
    where $p_\bpi(a|s)$ is the probability of taking action $a$ in state $s$ while following the policy mixture $\bpi$.
\end{assumption}
\vspace{5pt}
\begin{assumption}\label{asn:inherent_error}
    (Inherent Bellman Error) The following error bound holds for all $k$,
    \begin{equation*}
        \epsilon_\tn{approx} := \max\limits_{f\in\gF}\min\limits_{f'\in\gF}\|f'-\gT f\|^2_{2,\bpi_k}
    \end{equation*}
    where $\gT$ is the Bellman operator defined as $\gT f(s,a) := r(s,a)+\gamma\E_{s'\in P(.|s,a)} \max_{a'\in\gA} f(s',a')$.
\end{assumption}
\vspace{5pt}
\begin{lemma}\label{lemma:fqi_performance_difference}
        (proved in Supplementary \ref{appendix:fqi_performance_difference}) Performance of the policy obtained using FQI with estimated reward is $\epsilon_d$ sub-optimal with probability at least $1-(\delta_d+\delta)$ where,
\begin{align*}
    \epsilon_d & \leq \frac{\sqrt{C}}{(1-\gamma)^2} \sqrt{4\epsilon_\tn{approx}^2 + \frac{48\cdot214\cdot V_\tn{max}^4}{HN_T}\cdot\Psi(\delta) } + \frac{\gamma^{N_\tn{FQI}}V_{\max}}{1-\gamma} \\
    &\quad + \frac{1}{(1-\gamma)^2} \left(\gamma^H + \sqrt{\frac{\log \left(2/\delta_d\right)}{2N_T}}\right)
\end{align*}
and,
\begin{align*}
    \Psi(\delta) = \log\frac{28e(\text{dim}_\mc{F} + 1)N_\tn{FQI}}{\delta} + \text{dim}_\mc{F}\log (640HN_TV_{\max}^2)
\end{align*}
\end{lemma}

Lemma \ref{lemma:fqi_performance_difference} bounds the error of the batch RL step (solved via FQI) by decomposing it into the reward estimation error (bounded via Hoeffding’s inequality on $N_T$ trajectories) and the offline RL error.
\vspace{5pt}
\begin{theorem}\label{thm:discrete_bound}
Define the optimality gap as $h_k = \gF(\bpi^*) - \gF(\bpi_k)$ where $\bpi^* \in \argmax_{\bpi\in \bPi} \gF(\bpi)$. Then with probability greater than or equal to $1-(\delta_d + \delta)$,
\begin{align*}
    h(\bpi_K) \leq &\frac{\sqrt{C}}{(1-\gamma)^2} \sqrt{4\epsilon_\tn{approx}^2 + \frac{48\cdot214\cdot V_\tn{max}^4}{HN_T}\cdot\Psi(\delta/K) } + \frac{\gamma^{N_\tn{FQI}}V_{\max}}{1-\gamma} + \frac{2C_\gF}{K+1}\\
    &+ \frac{1}{(1-\gamma)^2} \left(\gamma^H + \sqrt{\frac{\log \left(2K/\delta_d\right)}{2N_T}}\right)
\end{align*}
\end{theorem}

\begin{proof}
    Follows directly from Lemmas \ref{lemma:wf_convergence} and \ref{lemma:fqi_performance_difference}.
\end{proof}

Theorem \ref{thm:discrete_bound} provides a PAC guarantee, decomposing the optimality gap into intuitive components. The term $O(1/K)$ reflects the convergence of the Frank-Wolfe optimization, which vanishes asymptotically. The remaining terms represent the performance floor determined by finite resources and structural choices: the statistical error scales with $O(1/\sqrt{N_T}) + O(\gamma^H) + O(\gamma^{N_\tn{FQI}}) + O(\sqrt{log(HN_T) \times N_\tn{FQI}/(HN_T)})$ due to finite sampling. Crucially, $\epsilon_{\text{approx}}$ denotes the Inherent Bellman Error which is a fixed structural property of the chosen action-value predictor class $\mathcal{F}$. Thus, the theorem guarantees that with a sufficiently expressive function class and adequate static budgets ($N_T, H, N_{\text{FQI}}$), the algorithm converges to a near-optimal solution with high probability.

\rev{It is worth noting that concentrability (Assumption \ref{asn:concentrability}) is practically realized through initial exploration and use of all trajectories accumulated till the execution of batch RL subroutine. Without accumulation, as $\bpi_{k-1}$ becomes deterministic, $p(s, a) \to 0$ for unvisited regions, causing $C \to \infty$ and loosening the error bounds. Thus, accumulation stabilizes $C$.}

\rev{\textbf{Continuous state-action spaces.} In continuous state spaces, diversity in goal state visitation is not just measured by how often different goal states were visited but also by how diverse the visited goal states themselves are. We adapt DDGC for continuous state-action spaces by using continuous version of FQI (also called Fitted Actor Critic - Algorithm \ref{alg:FAC}) proposed by \citet{antos2007fitted} and using state discretization for $\hat{d}_k$ estimation which is a valid procedure as it treats the total space as a collection of distinct, non-overlapping regions of measurable volume.}
A formal analysis of the continuous algorithm can be constructed in a manner that closely parallels the one for the discrete setting. A full exposition is omitted to avoid redundancy.

\section{Experiments}\label{section:experiments}

To evaluate the performance of algorithm empirically, we perform 2 sets of experiments. Experiments on controlled synthetic environments are provided in Section \ref{sec:controlled_synthetic_environments}. These are discrete MDPs designed to carefully study different aspects of the proposed algorithm and comparison with existing algorithms. Section \ref{sec:rl_benchmarks} contains experimental results on MuJoCo environments. Specifically, we perform experiments on 4 different Brax \citep{brax2021github} environments adapted for multi-goal settings by \citet{bortkiewicz2025accelerating}.

\subsection{Controlled Synthetic Environments}\label{sec:controlled_synthetic_environments}
\begin{wrapfigure}{r}{0.5\textwidth}
    \centering
    \includegraphics[width=\linewidth]{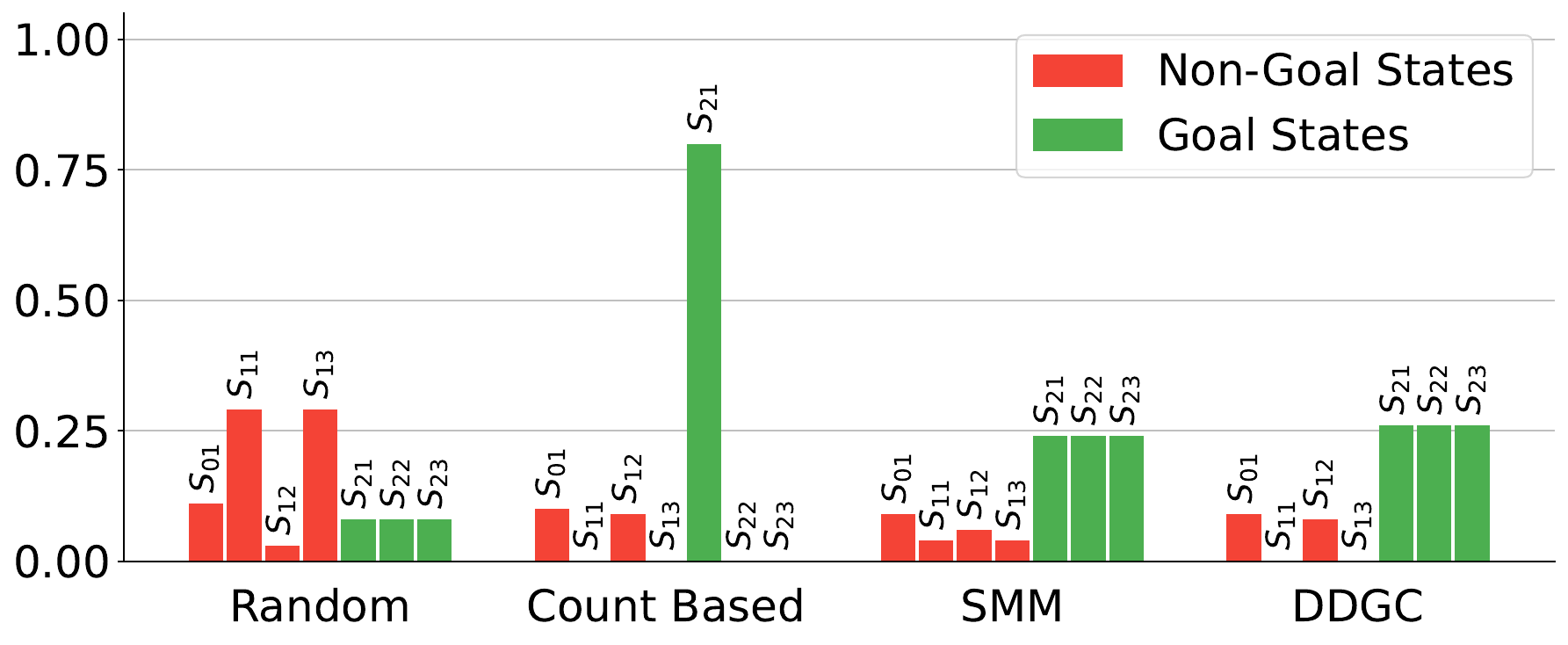}
    \caption{For the MDP in Figure \ref{figure:mdp}, we compare the discounted marginal state distribution induced by different algorithms over 7 states in the MDP. Each bar indicates the induced probability mass over a state. Green bars indicate discounted marginal state probability of different goal states and red bars indicate those of the non-goal states.}
    \label{fig:tabular_results}
\end{wrapfigure}

To empirically validate the importance of objective function, we conduct experiments for tabular settings on MDP defined in Figure \ref{figure:mdp}. The discounted marginal state distribution (probability mass function) induced by policies learnt via different algorithms can be visualized in Figure \ref{fig:tabular_results}.

DDGC achieves optimal performance by spreading uniform mass over goal states without compromising on the return (which is proportional to sum of probability mass over goal states). With count based exploration in Q-Learning, once a high return path to goal is found, the same is exploited with no consideration of visiting diverse goal states. A random policy on the other hand puts no consideration on finding a high return policy. The SMM objective encourages the agent to visit all states almost uniformly. More precisely, we use the target density function $d(s)=\exp(r(s))$ as proposed in \cite{lee2019efficient}. Clearly, SMM assigns a positive discounted marginal state distribution to each state, including the non-goal states which is not an optimal choice. While the performance difference between SMM and DDGC might not appear very significant in this small scale setup, once can imagine how this difference could become more prominent for complex environments. We defer the experimental details to Supplementary \ref{app:tabular_experiments}.


\begin{figure*}[ht]
    \centering
    \includegraphics[width=\linewidth]{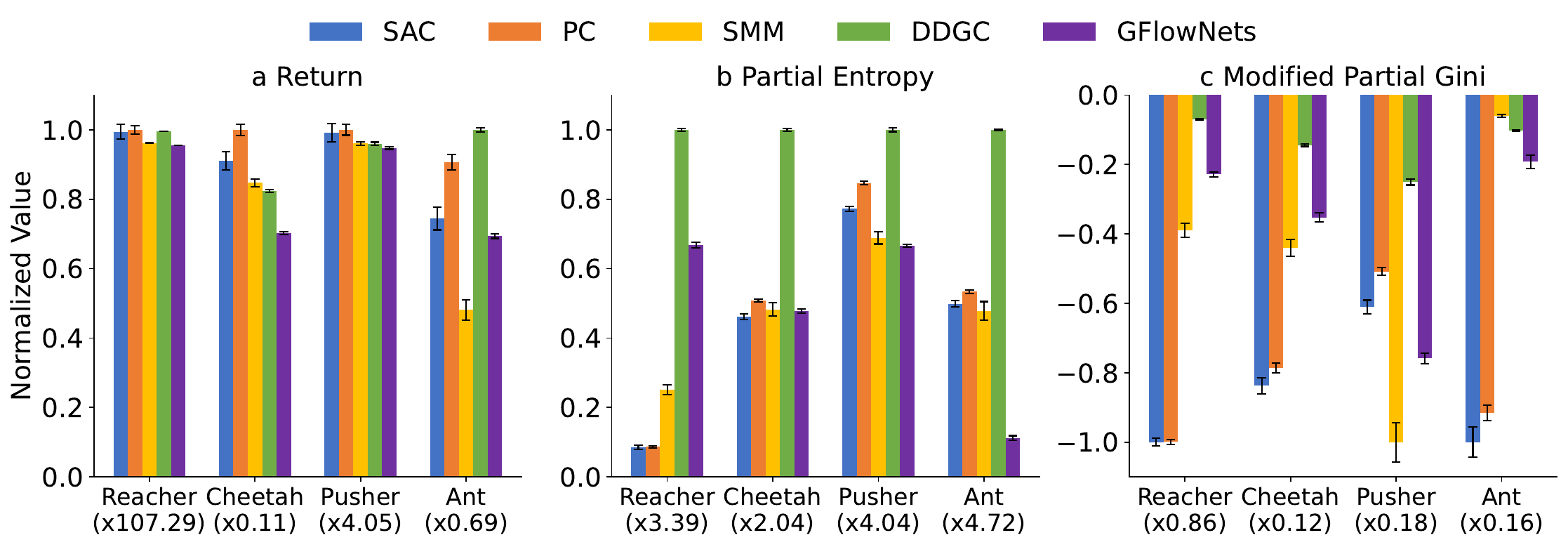}
    \caption{(a) Plot of discounted return from learnt policies; (b) Plot of partial entropy of discounted marginal state distribution over discretized state space of learnt policies; (c) Plot of modified partial gini criterion of discounted marginal state distribution over discretized state space of learnt policies. In each of the plots, values are averaged over 7 runs and normalized for each environment (normalization constant is present along with the x-label). Higher value is better for all 3 metrics.}
    \label{fig:rl_benchmarks}
\end{figure*}

\subsection{RL Benchmarks}\label{sec:rl_benchmarks}
\begin{wrapfigure}{r}{0.5\textwidth}
    \centering
    \includegraphics[width=\linewidth]{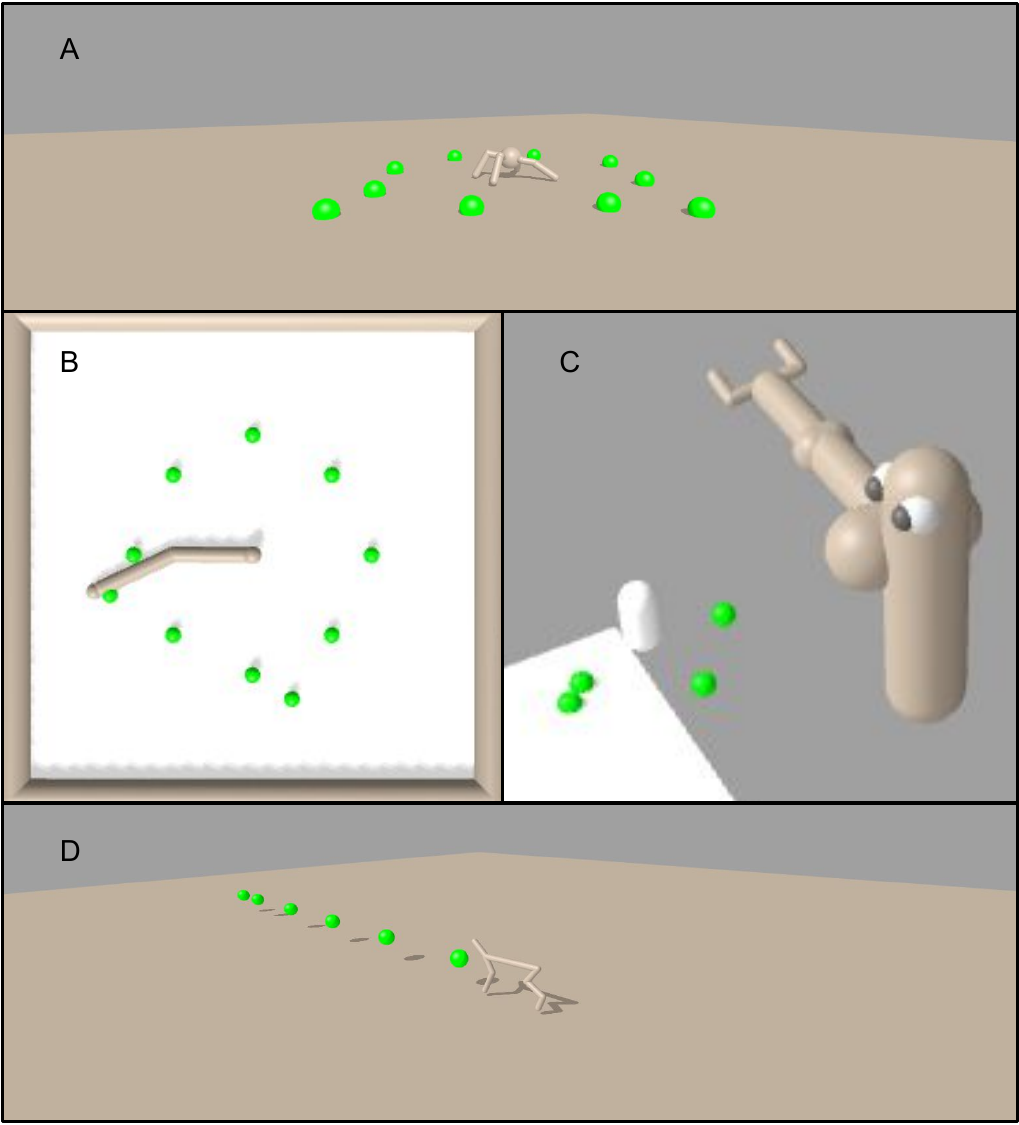}
    \caption{All 4 MuJoCo environments (A. ant, B. reacher, C. pusher, D. half cheetah) have multiple goal regions indicated by green spheres.}
    \label{fig:env_ss}
\end{wrapfigure}
We also evaluate our algorithm and benchmark against SAC, Pseudo Counts, SMM and GFlowNet baselines on 4 MuJoCo environments - reacher, pusher, ant and half cheetah environments. Pseudo Counts and SMM use SAC as the underlying RL algorithm similar to \citet{lee2019efficient}.
We use the multi goal setup in line with JaxGCRL \citep{bortkiewicz2025accelerating} with sparse rewards but without goal conditioning. In other words, rewards are binary and the observations do not contain any information about the goal. The MDP is multi-goal and the set of goal states comprises all states that satisfy a certain criteria (for example, distance from one of the pre-defined points is less than a certain threshold). We defer rest of the experimental details related to the environments and algorithm implementations to Supplementary \ref{app:jaxgcrl_experiments}.
Due to large state space, we cannot visualize the probability distribution induced over all the states unlike the synthetic MDP in Section \ref{sec:controlled_synthetic_environments}. Thus, we empirically compute the return and diversity in visited goal states for each policy. We roll out trajectories using learnt policies (or policy mixtures) with fixed horizon length. To estimate the discounted marginal state distribution, we discretize the state space similar to \citet{lee2019efficient}. We use the sampled trajectories to empirically compute the following statistics:

\begin{itemize}
    \item \textbf{Return.} This is the empirical estimation of discounted return by following a policy (or a mixture thereof).
    \item \textbf{Partial Entropy.} Entropy for a distribution $d(s)$ is defined as $\sum_{s\in\gS} -d(s)\log d(s)$. We, however are interested only in the diversity over goal states. So, the quantity of interest is $\sum_{s\in\gS^+} -d(s)\log d(s)$ which we call the partial entropy. We compute partial entropy of discounted marginal state distribution over discretized state space.
    \item \textbf{Modified Partial Gini Criterion.} Gini Criterion for a distribution $d(s)$ is usually defined as $1-\sum_{s\in\gS}d(s)^2$. We, however are interested only in the diversity over goal states. So, the quantity of interest is $1-\sum_{s\in\gS^+}d(s)^2$ which we call partial gini criterion. For better visualization, we get rid of the constant and compute $-\sum_{s\in\gS^+}d(s)^2$ instead and call it modified partial gini criterion.
\end{itemize}

\textit{Return} is a direct indicator of how densely the goals are covered by a policy. For diversity, different metrics can be used. We use two such metrics, \textit{Partial Entropy} and \textit{Modified Partial Gini Criterion}. The experimental results and comparison with baselines is visualized in Figure \ref{fig:rl_benchmarks}.

DDGC consistently matches the best performing algorithm in terms of return. At the same time, DDGC also achieves a policy with high diversity in goal states visited as is indicated by Partial Entropy as well as Modified Partial Gini Criterion. Further, DDGC is computationally efficient to run as shown by the average wall clock time taken to run experiments reported in Supplementary \ref{app:compreq}.

\section{Conclusion}\label{section:conclusion}

Our work studies the problem of learning a policy mixture for multi-goal RL to frequently visit the goal states but also visit a diverse set of goal states. With this motivation, we design an objective function that maximizes the goal state visitation frequency as well the diversity of goal states visited. %
We optimize this objective function using the Frank-Wolfe method, developing novel algorithm which solves an offline RL subproblem at each step.
 We provide theoretical guarantees on convergence of the algorithm and we extend the algorithm for continuous state-action spaces.
We also evaluate the algorithm on a small synthetic MDP to highlight the importance of the proposed objective function. We also perform experiments on MuJoCo environments which clearly demonstrate the improvement in goal coverage without losing much on the frequency of visitation. As our algorithm relies on multiple calls to RL subroutine (as is also the case with existing works on distribution matching), this leads to an increased computational requirement for training. Future works in this direction can further advance this area with advent of more efficient algorithms for reaching diverse goals.








\bibliography{main}
\bibliographystyle{rlj}

\beginSupplementaryMaterials
\section{Balancing Return and Goal State Dispersion}\label{appendix:balancing_conflict}

The return maximization objective and goal state dispersion objective may not necessarily align.
Particularly, we highlight the following two underlying reasons that may lead to this misalignment:
\begin{enumerate}
    \item \textbf{Discounting.} Because of discounting, the return maximization objective  encourages visitation of goal states that are closer to the start states over the ones that are farther. This is demonstrated via a simple example in Figure \ref{figure:mg_conflict}. For a discount factor value $<1$, the highest expected return would be given by a policy that takes action $A_1$ in state $S_{01}$. Maximizing the proposed objective, however, leads to a policy (or a mixture) that visits all the goal states with non-zero probability. Further, for $\gamma=1$, it can be observed that the objective is maximized when all goal states are visited equally likely as the objectives align in this case.
    \item \textbf{Environment Dynamics.} The return maximization and diverse goal visitation objectives may also conflict because of the environment dynamics. Figure \ref{figure:mg2_conflict} illustrates this with an example. In this example, even with $\gamma=1$, the two objectives do not align. While the highest return is obtained by exploiting smaller loop comprising only the goal states, a more diverse goal coverage can be obtained by also visiting goal states in the loop which contains a non-goal state.
\end{enumerate}

\begin{figure}[ht]
    \centering
    \begin{subfigure}[t]{0.3\textwidth}
        \includegraphics[width=\textwidth]{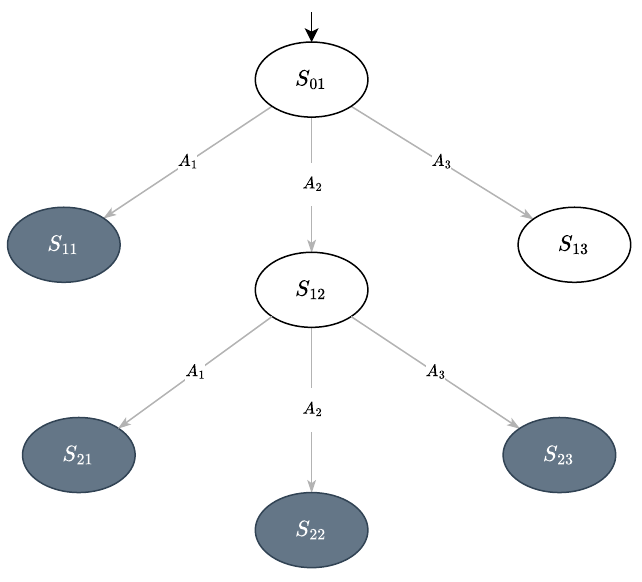}
        \caption{}
    \end{subfigure}
    \hfill
    \begin{subfigure}[t]{0.3\textwidth}
        \includegraphics[width=\linewidth]{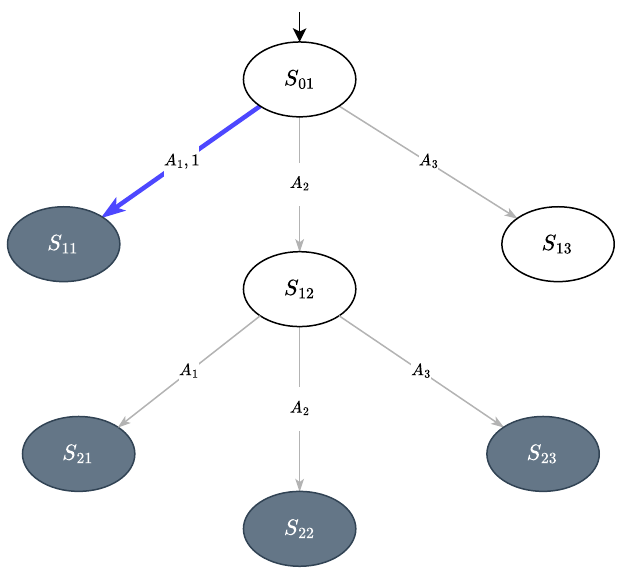}
        \caption{}
    \end{subfigure}
    \hfill
    \begin{subfigure}[t]{0.3\textwidth}
        \includegraphics[width=\linewidth]{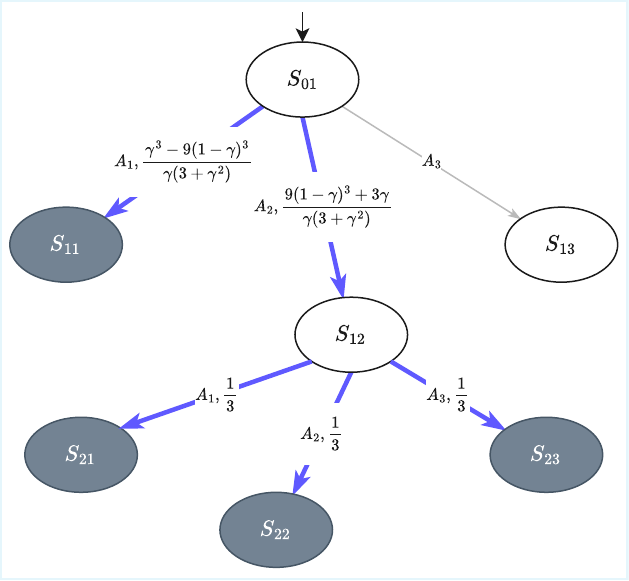}
        \caption{}
    \end{subfigure}

    \caption{(a) A continuing deterministic MDP with a start state, and highlighted states as the goal states. The terminal states act as sink states i.e. any action taken in these states leads back to the same state. (b) Marginal state distribution with the highest expected return as it is concentrated on the closest goal state to avoid higher discounting penalty. (c) Marginal state distribution that is well dispersed across goal states but compromises on the expected return by incurring larger discounting penalty.}\label{figure:mg_conflict}
\end{figure}

\begin{figure}[ht]
    \centering
    \begin{subfigure}[t]{0.3\textwidth}
        \includegraphics[width=\textwidth]{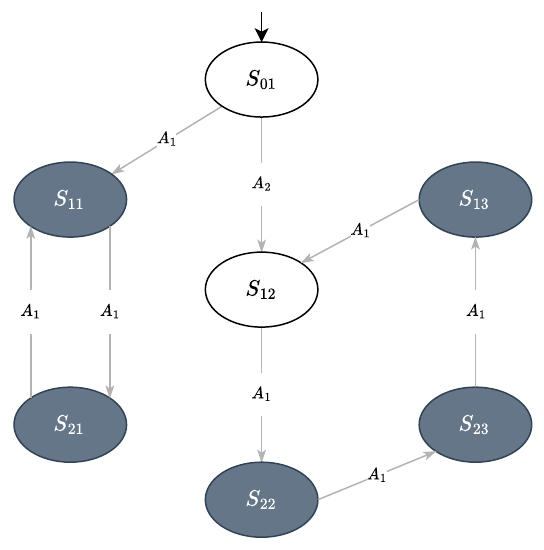}
        \caption{}
    \end{subfigure}
    \hfill
    \begin{subfigure}[t]{0.3\textwidth}
        \includegraphics[width=\linewidth]{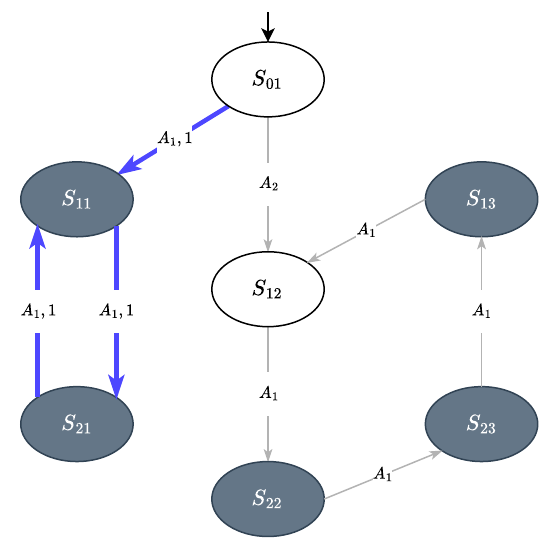}
        \caption{}
    \end{subfigure}
    \hfill
    \begin{subfigure}[t]{0.3\textwidth}
        \includegraphics[width=\linewidth]{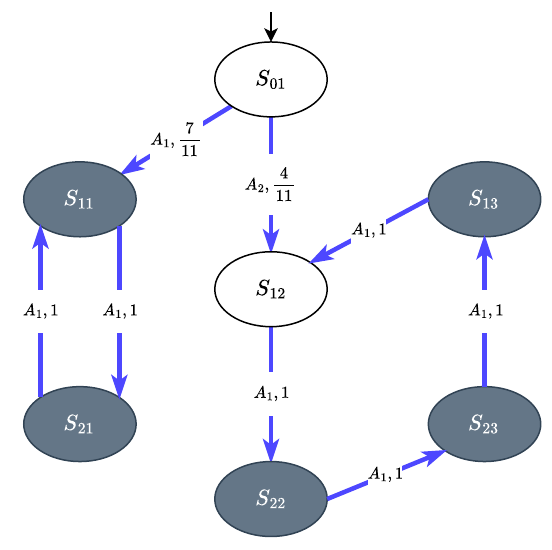}
        \caption{}
    \end{subfigure}

    \caption{(a) A continuing deterministic MDP with a start state, and highlighted states as the goal states. (b) Marginal state distribution with the highest expected return as it is concentrated densely on the goal state by avoiding visitation on non-goal state. (c) Marginal state distribution that is obtained for $\gamma=1$ by optimizing for the DDGC objective is well dispersed across goal states but compromises on expected return by visiting one non-goal state.}\label{figure:mg2_conflict}
\end{figure}

\section{Proof of Lemma \ref{lemma:objective_properties}}\label{appendix:objective_properties}

\textbf{1.} $\bPi$ is a convex set.

\begin{proof}
    $\bPi$ is convex as a convex combination of two valid probability distributions over $\Pi$ is another valid probability distribution over $\Pi$ and thus belongs to $\bPi$.
\end{proof}

\textbf{2.} $\gF$ is a concave function over $\bPi$.

\begin{proof}
Define $d(s) = \bd[\bpi](s)$. Let $\gG(d) = \sum_{s\in\gS^+} \left[d(s) - \frac{1}{2}d(s)^2\right]$.

We first show that $\gG$ is concave over the space of all valid probability mass functions $d\in \Delta \gS$ by showing that the Hessian of $\gG$ is negative semi-definite.

\begin{equation*}
    \frac{\delta G}{\delta d(s_i)} = R(s_i)\left[1-d(s_i)\right]
\end{equation*}

\begin{equation*}
    H_{ij} = \frac{\delta^2 G}{\delta d(s_i)\delta d(s_j)} = \begin{cases}
 -R(s_i)& \text{if } s_i = s_j \\\\
 
 0 & \text{if } s_i \neq s_j \end{cases}
\end{equation*}

For any non-zero $v \in \sR^{|\gS|}$,
\begin{align*}
    v^T H_{ij} v &= \sum\limits_i\sum\limits_j v_i H_{ij} v_j = \sum\limits_i v_i^2 H_{ii} + \sum\limits_{i\neq j} v_i H_{ij} v_j = -\sum\limits_i v_i^2 R(s_i) \leq 0
\end{align*}

Thus, $v^T H_{ij} v \leq 0$. So, $H$ is negative semi-definite.

\begin{align*}
    \gF(\bpi_1) &= \gG(\bd[\bpi_1]) = \gG(d_1)\\
    \gF(\bpi_2) &= \gG(\bd[\bpi_2]) = \gG(d_2) \\
    \gF(\alpha\bpi_1 + (1-\alpha)\bpi_2) &= \gG(\alpha\bd[\bpi_1] + (1-\alpha)\bd[\bpi_2]) = \gG(\alpha d_1 + (1-\alpha) d_2)
\end{align*}

Since, we have shown $\gG$ to be concave, from the above argument it follows that $\gF$ is concave as well.

\end{proof}

\textbf{3.} $C_\gF = \sup\limits_{\substack{\bpi_1,\bpi\in\bPi\\\lambda\in(0,1)\\\bpi_2=\bpi_1+\lambda(\bpi-\bpi_1)\\
}} \frac{2}{\lambda^2}\left[\gF(\bpi_1) + \langle\bpi_2-\bpi_1, \nabla_{\pi}\gF(\bpi_1)\rangle -\gF(\bpi_2)\right] < \infty$

\begin{equation*}
    \frac{\delta\gF}{\delta\bpi(\pi)} = \sum\limits_{s\in\gS^+} d[\pi](s)\left(1-\bd[\bpi](s)\right)
\end{equation*}

Using the expression for functional derivative to find the required dot product,

\begin{align*}
    \langle\bmu_, \nabla_{\pi}\gF(\bpi)\rangle 
    &= \int\limits_{\Pi} \sum\limits_{s\in\gS^+} d[\pi](s)(1-\bd[\bpi](s)) \bmu(\pi) d\pi = \sum\limits_{s\in\gS^+} (1-\bd[\bpi](s)) \int\limits_{\Pi} d[\pi](s) \bmu(\pi) d\pi \\
    &= \sum\limits_{s\in\gS^+} \bd[\bmu](s) (1-\bd[\bpi](s)) = \sum\limits_{s\in\gS^+} \bd[\bmu](s) - \sum\limits_{s\in\gS^+} \bd[\bmu](s)\bd[\bpi](s)
\end{align*}

\begin{equation}\label{eqn:dot_product}
    \langle\bmu_, \nabla_{\pi}\gF(\bpi)\rangle = \sum\limits_{s\in\gS^+} \bd[\bmu](s) - \sum\limits_{s\in\gS^+} \bd[\bmu](s)\bd[\bpi](s)
\end{equation}

Using expression for dot product and definition of $\gF$,

\begin{align*}
    \gF(\bpi_1) + \langle\bpi_2-\bpi_1, \nabla_{\pi}\gF(\bpi_1)\rangle -\gF(\bpi_2) &= \sum\limits_{s\in\gS^+} (\bd[\bpi_1](s)-\bd[\bpi_2](s)) \\
    &\quad\quad\quad\quad + \frac{1}{2}\sum\limits_{s\in\gS^+} (\bd[\bpi_2](s)^2-\bd[\bpi_1](s)^2) \\
    &\quad\quad\quad\quad + \sum\limits_{s\in\gS^+} (\bd[\bpi_2](s) - \bd[\bpi_1](s)) \\
    &\quad\quad\quad\quad + \sum\limits_{s\in\gS^+} (\bd[\bpi_1](s)^2 - \bd[\bpi_1](s)\bd[\bpi_2](s))\\
    &= \frac{1}{2}\sum\limits_{s\in\gS^+} (\bd[\bpi_2](s) - \bd[\bpi_1](s))^2
\end{align*}

As $\bpi_2 - \bpi_1 = \lambda(\bpi-\bpi_1)$, we have,

\begin{align*}
    \bd[\bpi_2](s) - \bd[\bpi_1](s) = \lambda(\bd[\bpi](s) - \bd[\bpi_1](s))
\end{align*}

Squaring and summing over all positive states, we have

\begin{align*}
    \sum\limits_{s\in\gS^+}(\bd[\bpi_2](s) - \bd[\bpi_1](s))^2 = \lambda^2 \sum\limits_{s\in\gS^+} (\bd[\bpi](s) - \bd[\bpi_1](s))^2 \leq 2\lambda^2
\end{align*}

Re-arranging we have,

\begin{align*}
    \frac{1}{\lambda^2} \leq \frac{2}{\sum\limits_{s\in\gS^+}(\bd[\bpi_2](s) - \bd[\bpi_1](s))^2}
\end{align*}

Combining above results, we have

\begin{align*}
    \frac{2}{\lambda^2}\left[\gF(\bpi_1) + \langle\bpi_2-\bpi_1, \nabla_{\pi}\gF(\bpi_1)\rangle -\gF(\bpi_2)\right] \leq 1
\end{align*}

Thus, $C_\gF \leq 1$.

\section{Proof of Lemma \ref{lemma:wf_convergence}}\label{appendix:wf_convergence}
\begin{proof}
Define $d_k = \bd[\bpi_k]$ and $\gG(d) = \sum_{s\in\gS^+} \left[d(s) - (1/2)d(s)^2\right]$.
Note that $\gG(\bd[\bpi]) = \gF(\bpi)$.

The update step in algorithm is $\bpi_k = (1-\lambda_k)\bpi_{k-1} + \lambda_k\bmu_k = \bpi_{k-1} + \lambda_k(\bmu_k-\bpi_{k-1})$.
This gives us $\bpi_k - \bpi_{k-1} = \lambda_k(\bmu_k-\bpi_{k-1})$.

From Lemma \ref{lemma:objective_properties}, we have $\gF(\bpi_k)-\gF(\bpi_{k-1}) \geq \langle\bpi_k-\bpi_{k-1}, \nabla_\pi\gF(\bpi_{k-1})\rangle - \frac{C_\gF}{2}\lambda_k^2$.

Following the proof structure in \citet{pmlr-v28-jaggi13}, adding and subtracting $\gF(\bpi^*)$ on LHS and multiplying both sides by -1, we have,
\begin{equation}\label{eqn:curvature_ineq}
    h(\bpi_k)-h(\bpi_{k-1}) \leq \frac{C_\gF}{2}\lambda_k^2 - \lambda_k\langle\bmu_k-\bpi_{k-1}, \nabla_\pi\gF(\bpi_{k-1})\rangle
\end{equation}

From Equation \ref{eqn:dot_product},
\begin{align*}
    \langle\bmu_, \nabla_{\pi}\gF(\bpi)\rangle = \sum\limits_{s\in\gS^+} \bd[\bmu](s) - \sum\limits_{s\in\gS^+} \bd[\bmu](s)\bd[\bpi](s)
\end{align*}

Define $r_k(s)$  as follows:
\begin{equation}\label{eqn:reward}
    r_k(s) = 
    \begin{cases}
    1-\bd[\bpi_{k-1}](s), & s\in\gS^+\\
    0, & s\in\gS^-
    \end{cases}
\end{equation}

So we have,
\begin{equation}\label{eqn:rl_performance}
    \langle\bmu_k, \nabla_\pi\gF(\bpi_{k-1})\rangle = \sum\limits_{s\in\gS} \bd[\bmu_k](s) r_k(s)
\end{equation}

Since $\bPi$ is a convex set as established in \ref{lemma:objective_properties}, $\exists \mu\in \Pi$ such that the policy mixture containing only $\mu$  maximizes $\langle\bmu, \nabla_\pi\gF(\bpi_{k-1})\rangle$ because policy mixtures containing a single policy from set $\Pi$ form the extreme points of the convex set $\bPi$.

Note that for a policy mixture $\bmu_k = \{(\mu_k, 1)\}$, $\forall s\in\gS$, $\bd[\bmu_k](s)=d[\mu_k](s)$. And maximizing $\langle\bmu, \nabla_\pi\gF(\bpi_{k-1})\rangle$ boils down to solving the following problem

\begin{equation*}
    \argmax\limits_{\mu_k\in\Pi}\sum\limits_{s\in\gS} d[\mu_k](s) r_k(s)    
\end{equation*}

Note that this is equivalent to solving an RL problem where the MDP is the same as original MDP but with $r_k$ as the reward function instead. This emergent structure has been previously exploited in \citet{hazan2019provably}.
$\langle\bmu, \nabla_\pi\gF(\bpi_{k-1})\rangle$ is the expected return of a policy mixture $\bmu$ w.r.t. reward function $r_k$. Let $\bmu_k^*$ be the optimal policy mixture for the reward function $r_k$.
Assuming that the solution given by RL algorithm is at most $\epsilon_k$ sub-optimal with probability $1-\delta_k$, we have

\begin{align*}
    \langle\bmu_k, \nabla_\pi\gF(\bpi_{k-1})\rangle + \epsilon_k &\geq \langle\bmu_k^*, \nabla_\pi\gF(\bpi_{k-1})\rangle\\
    \langle\bmu_k, \nabla_\pi\gF(\bpi_{k-1})\rangle + \epsilon_k &\geq \langle\bmu, \nabla_\pi\gF(\bpi_{k-1})\rangle \quad \forall \bmu\in\bPi \\
    \langle\bmu_k, \nabla_\pi\gF(\bpi_{k-1})\rangle + \epsilon_k &\geq \langle\bpi^*, \nabla_\pi\gF(\bpi_{k-1})\rangle
\end{align*}

From Equation \ref{eqn:curvature_ineq} and \ref{eqn:rl_performance}, we have

\begin{align*}
    h(\bpi_k)-h(\bpi_{k-1}) \leq \frac{C_\gF}{2}\lambda_k^2 - \lambda_k\langle\bmu_k-\bpi_{k-1}, \nabla_\pi\gF(\bpi_{k-1})\rangle
\end{align*}

With probability $1-\delta_k$,

\begin{align*}
    h(\bpi_k)-h(\bpi_{k-1}) \leq \frac{C_\gF}{2}\lambda_k^2 - \lambda_k\langle\bpi^*-\bpi_{k-1}, \nabla_\pi\gF(\bpi_{k-1})\rangle + \lambda_k\epsilon_k
\end{align*}

Since $\gF$ is concave, $\gF(\bpi^*)-\gF(\bpi_{k-1}) \leq \langle\bpi^*-\bpi_{k-1}, \nabla_\pi\gF(\bpi_{k-1})\rangle$. Hence,

\begin{align*}
    h(\bpi_k)-h(\bpi_{k-1}) &\leq \frac{C_\gF}{2}\lambda_k^2 - \lambda_k (\gF(\bpi^*)-\gF(\bpi_{k-1})) + \lambda_k\epsilon_k\\
    h(\bpi_k)-h(\bpi_{k-1}) &\leq \frac{C_\gF}{2}\lambda_k^2 - \lambda_k h(\bpi_{k-1}) + \lambda_k\epsilon_k\\
    h(\bpi_k) &\leq (1 - \lambda_k) h(\bpi_{k-1}) + \frac{C_\gF}{2}\lambda_k^2 + \lambda_k\epsilon_k
\end{align*}

For $\lambda_k = \frac{2}{k+1}$,

\begin{align*}
    h(\bpi_k) - \frac{k-1}{k+1} h(\bpi_{k-1}) &\leq \frac{2C_\gF}{(k+1)^2} + \frac{2\epsilon_k}{k+1}
\end{align*}

Multiplying $k(k+1)$ on both sides,

\begin{align*}
    k(k+1)h(\bpi_k) - (k-1)k h(\bpi_{k-1}) &\leq \frac{2kC_\gF}{k+1} + 2k\epsilon_k
\end{align*}

Summing for $k = 1\cdots K$ by union bound and Boole's inequality, with probability $\geq 1- \sum\limits_{k=1}^K \delta_k$,

\begin{align*}
    K(K+1)h(\bpi_K) &\leq 2\sum\limits_{k=1}^K \left(\frac{kC_\gF}{k+1} + k\epsilon_k\right)\\
    &\leq 2\sum\limits_{k=1}^K \left(C_\gF + k\epsilon_k\right)\\
    h(\bpi_K) &\leq \frac{2C_\gF}{K+1} + \frac{2}{K(K+1)}\sum\limits_{k=1}^K k\epsilon_k
\end{align*}
\end{proof}

\section{Proof of Lemma \ref{lemma:fqi_performance_difference}}\label{appendix:fqi_performance_difference}

\begin{proof}

Let $\mu_k$ be the policy obtained using FQI with rewards estimated with $\hat{d}_k$.
Since, FQI is an iterative algorithm, let the policy obtained at $i^{th}$ iteration be denoted by $\mu_k^i$ and the corresponding Q-value function by $f_i$.
Also, denote the optimal policy for MDP with true reward function as defined in Equation \ref{eqn:reward} by $\mu_k^*$ and the corresponding Q-value function by $Q^*$.
We use $V(\pi)$ to denote the performance of a policy $\pi$ for MDP with true reward as defined in Equation \ref{eqn:reward}.
Note that there are two sources of error here, 1) our algorithm optimizes for a pseudo-reward which is essentially an estimate of the true reward, and 2) inherent error because of FQI.
Both these factors contribute to the sub-optimality of $\mu_k$.
We denote the discounted marginal state distribution induced by the policy $\pi$ under aforementioned MDP as $d^{\pi}$.

Using performance difference lemma from \cite{agarwal2019reinforcement},

\begin{align*}
    (1-\gamma)(V(\mu_k^*) - V(\mu_k^i)) &= \E_{s\sim d^{\mu_k^i}} \left[Q^*(s,\mu^*(s))-Q^*(s,\mu_k^i(s))\right]\\
    &\leq \E_{s\sim d^{\mu_k^i}} \left[Q^*(s,\mu^*(s)) -f_i(s,\mu^*(s)) +f_i(s,\mu_k^i(s)) -Q^*(s,\mu_k^i(s))\right]\\
    &\leq \|Q^*-f_i\|_{1,d^{\mu_k^i}\odot \mu^*} + \|Q^*-f_i\|_{1,d^{\mu_k^i}\odot \mu^i_k}\\
    &\leq \|Q^*-f_i\|_{2,d^{\mu_k^i}\odot \mu^*} + \|Q^*-f_i\|_{2,d^{\mu_k^i}\odot \mu^i_k}
\end{align*}

Consider state-action distribution $\beta$ induced by some policy.

Let $\beta^{i-1}(s',a') = \sum\limits_{s,a} \beta^{i}(s,a) P(s'|s,a) \1\{a'=\argmax_a \left[Q^*(s',a)-f_{i-1}(s',a)\right]^2\}$. Also, let $\beta_P^{i}(s')=\sum\limits_{s,a} \beta^{i}(s,a)P(s'|s,a)$. We have

\begin{align*}
    \|Q^*-f_i\|_{2,\beta} &\leq \|Q^*-\gT f_{i-1}\|_{2,\beta} + \|f_i-\gT f_{i-1}\|_{2,\beta}\\
    &\leq \sqrt{\operatorname{\E}\limits_{\substack{s,a\sim \beta\\s'\sim P(\cdot|s,a)}} \left[\substack{R(s')\left(\bd[\bpi_k](s') - \hat{d}_k(s')\right) \\+\gamma\left(\max_{a'} Q^*(s',a') - \max_{a'} f_{i-1}(s',a')\right)}\right]^2} + \sqrt{C}\|f_i-\gT f_{i-1}\|_{2,\nu}\\
    &\leq \|R\odot(r_k-\hat{r}_k)\|_{2,\beta_P} + \gamma\|Q^*-f_{i-1}\|_{2,\beta'} + \sqrt{C}\|f_i-\gT f_{i-1}\|_{2,\nu}
\end{align*}

Let $\beta^i = \beta$ and $\beta_P^i=\beta_P$,

\begin{align*}
    \|Q^*-f_i\|_{2,\beta^i} &- \gamma\|Q^*-f_{i-1}\|_{2,\beta^{i-1}} \leq \|R\odot(r_k-\hat{r}_k)\|_{2,\beta_P^i} + \sqrt{C}\|f_i-\gT f_{i-1}\|_{2,\nu}\\
    \cdots\\
    \|Q^*-f_1\|_{2,\beta^1} &- \gamma\|Q^*-f_0\|_{2,\beta^0} \leq \|R\odot(r_k-\hat{r}_k)\|_{2,\beta_P^1} + \sqrt{C}\|f_1-\gT f_0\|_{2,\nu}\\
\end{align*}

Multiplying $j^{th}$ inequality by $\gamma^{j-1}$ and summing up the above inequalities,

\begin{align*}
    \|Q^*-f_i\|_{2,\beta} &\leq \sum\limits_{j=1}^i \gamma^{i-j} \|R\odot(r-\hat{r})\|_{2,\beta_P^j} + \sqrt{C}\sum\limits_{j=1}^{i}\gamma^{j-1}\|f_{i-j+1}-\gT f_{i-j}\|_{2,\nu} \\
    &\quad\quad + \gamma^i\|Q^*-f_0\|_{2,\beta^0}
\end{align*}

We will now show $\|R\odot(r-\hat{r})\|_{2,\alpha} \leq \epsilon_d$.

We know that $r(s)-\hat{r}(s) = 0$ for $s \notin \gS^+$. For simplicity, denote $\hat{d}_k$ by $\hat{d}$ and $\bd[\bpi_k]$ by $d$.
For $s \in \gS^+$, we have
\begin{align*}
    |r(s) - \hat{r}(s)| = |\hat{d}(s) - d(s)|
\end{align*}

In Supplementary \ref{appendix:error_in_d_approximation}, we show that $|\hat{d}(s) - d(s)| \le \epsilon_d$ with probability $\geq 1-\delta_d$. 

So, with probability $\geq 1-\delta_d$,
\begin{align*}
    |r(s) - \hat{r}(s)| &\leq \epsilon_d
\end{align*}

As the above inequality holds for all $s\in \gS^+$,
\begin{align*}
    \|r(s) - \hat{r}(s)\|_{\infty} \leq \epsilon_d
\end{align*}

So, we have that with probability $\geq 1-\delta_d$, $\|R\odot(r-\hat{r})\|_{2,\alpha} \leq \|R\|_{\infty}\cdot\|(r-\hat{r})\|_{\infty} \leq \epsilon_d$ and thus assuming $\|f_{i-t+1}-\gT f_{i-t}\|_{2,\nu} \leq \omega$ for all $t \in \{0,\dots, i\}$, we have,

\begin{equation}\label{eqn:bound1}
    \|Q^*-f_i\|_{2,\beta} \leq \frac{\epsilon_d + \sqrt{C}\omega}{1-\gamma} + \gamma^i V_{max}
\end{equation}

Following the pseudo-dimension based analysis of FQI in Supplementary F.2 of  \citet{LVY19}, we bound $\omega$ with high probability. Specifically, we use the result from Equation 35 of \citet{LVY19},

\begin{align*}
    P\left\{\|f_j-\gT f_{j-1}\|_{2,\nu} > \epsilon + 4\epsilon_\tn{approx}^2\right\} \leq \xi(\epsilon)\\
\end{align*}
where,
\begin{align*}
    \xi(\epsilon) &= 14\cdot e\cdot (\tn{dim}_\mc{F} + 1) \left(\frac{640V_{\max}^2}{\epsilon}\right)^{\tn{dim}_\mc{F}}\cdot \tn{exp}\left(-\frac{\epsilon HN_T}{48\cdot 214V_{\max}^4}\right) + \tn{exp}\left(-\frac{3}{416}\frac{\epsilon HN_T}{V_{\max}^2}\right)
\end{align*}

Since $V_{\max} \geq 1$,
\begin{align*}
    -\frac{\epsilon HN_T}{48\cdot 214V_{\max}^4} \geq -\frac{3}{416}\frac{\epsilon HN_T}{V_{\max}^2}
\end{align*}

Hence,
\begin{align*}
    P\left\{\|f_j-\gT f_{j-1}\|_{2,\nu} - 4\epsilon_\tn{approx}^2 > \epsilon\right\}& \\
    & \!\!\!\!\!\!\!\!\!\!\!\!\!\!\!\!\!\!\!\!\! \leq 28\cdot e\cdot (\tn{dim}_\mc{F} + 1) \left(\frac{640V_{\max}^2}{\epsilon}\right)^{\tn{dim}_\mc{F}}\cdot \tn{exp}\left(-\frac{\epsilon HN_T}{48\cdot 214V_{\max}^4}\right)
\end{align*}

Consider $\delta$ such that,
\begin{align*}
    \delta \geq 28\cdot e\cdot (\tn{dim}_\mc{F} + 1) \left(\frac{640V_{\max}^2}{\epsilon}\right)^{\tn{dim}_\mc{F}}\cdot \tn{exp}\left(-\frac{\epsilon HN_T}{48\cdot 214V_{\max}^4}\right)
\end{align*}

Taking negative log on both sides, we obtain the following
\begin{align*}
    \log\left(\frac{1}{\delta}\right) \leq \frac{\epsilon HN_T}{48\cdot 214 V_\tn{max}^4} - \log\left[28e(\tn{dim}_\mc{F} + 1)(640 V_{\max}^2)^{\tn{dim}_\mc{F}}\right] + \tn{dim}_\mc{F} \log\epsilon = \chi(\epsilon)
\end{align*}

For the following value of $\epsilon$, we have $\chi(\epsilon) \geq \log(1/\delta)$.
\begin{equation}\label{eqn:eps}
    \epsilon = \frac{48\cdot 214 V_{\max}^4}{HN_T}\left[\log\frac{28e(\tn{dim}_\mc{F}+1)}{\delta} + \tn{dim}_\mc{F}\log(640V_\tn{max}^2HN_T)\right]
\end{equation}

Thus, with probability $\geq 1-\delta$ we have,
\begin{equation}\label{eqn:bound2}
    \|f_j-\gT f_{j-1}\|_{2,\nu} < \sqrt{4\epsilon_\tn{approx}^2 + \epsilon}
\end{equation}

Using Expressions \ref{eqn:eps}, \ref{eqn:bound1} and \ref{eqn:bound2}, we can conclude that the following holds with probability $\geq 1-(\delta+\delta_d)$
\begin{align*}
    \|Q^*-f_{N_\text{FQI}}\|_{2,\beta} & \quad  \leq \quad \frac{\gamma^{N_{FQI}}V_{\max}}{1-\gamma} + \frac{1}{(1-\gamma)^2} \left(\gamma^H + \sqrt{\frac{\log \left(2/\delta_d\right)}{2N_T}}\right) \\
    &\!\!\!\!\!\!\!\!\!\!\!\!\!\!\!\!\!\!\!\!\!\!\!\!\!\!\!\!\!\!\!\!\!\!\!\!\!\! + \frac{\sqrt{C}}{(1-\gamma)^2} \sqrt{4\epsilon_\text{approx}^2 + \frac{48\cdot214\cdot V_\text{max}^4}{HN_T}\left[\log\frac{28e(\text{dim}_\mc{F} + 1)N_\text{FQI}}{\delta} + \text{dim}_\mc{F} \cdot \log (640HN_TV_{\max}^2)\right]}
\end{align*}
\end{proof}

\section{Error in discounted frequency approximation}\label{appendix:error_in_d_approximation}

Consider the following distributions,
\begin{align*}
    d(s)&=(1-\gamma)\E\left[\sum\limits_{t=0}^{\infty} \gamma^t P(s_t=s|s_0\sim \rho_0, \pi)\right]\\
    \hat{d}(s)&=(1-\gamma) \sum\limits_{i=1}^{N_T} \sum \limits_{t=0}^{H-1} \gamma^t I(s_{i,t}=s)\\
    d_H(s) &= (1-\gamma)\E \left[\sum\limits_{t=0}^{H-1} \gamma^t P(s_t=s|s_0\sim \rho_0, \pi)\right]
\end{align*}

Using triangle inequality,
\begin{align*}
    |\hat{d}(s)-d(s)| \leq |\hat{d}(s)-d_H(s)| + |d_H(s)-d(s)|
\end{align*}

We bound each of the terms on RHS,
\begin{align*}
    |d_H(s) - d(s)| &= (1-\gamma) \left|\sum\limits_{t=0}^{H-1} \gamma^t P(s_t=s|s_0\sim \rho_0,\pi)-\sum\limits_{t=0}^\infty \gamma^t P(s_t=s)\right|\\
    &= (1-\gamma)\left|\sum\limits_{t=H}^\infty P(s_t=s)\right|\\
    &\leq (1-\gamma)\frac{\gamma^H}{1-\gamma} = \gamma^H
\end{align*}

$|\hat{d}(s)-d_H(s)|$ is deviation of sample mean from true expectation of random variable $X_i(s) = (1-\gamma) \sum\limits_{t=0}^{H-1} \gamma^t I(s_{i,t}=s)$. Using Hoeffdings' inequality,
\begin{align*}
    P\left(\left|\hat{d}(s)-d_H(s)\right|\geq \sqrt{\frac{\log 2/\delta_d}{2N_T}}\right) \leq \delta_d
\end{align*}

Thus, with probability $\geq 1-\delta_d$, $\left|\hat{d}(s)-d_H(s)\right| \leq \sqrt{\frac{\log 2/\delta_d}{2N_T}}$.

Combining, the following holds with probability $\geq 1-\delta_d$
\begin{align*}
    \left|\hat{d}(s)-d(s)\right| \leq \gamma^H + \sqrt{\frac{\log 2/\delta_d}{2N_T}}
\end{align*}

\section{Experimental Details (Tabular)}\label{app:tabular_experiments}
For the simple MDP, we implement a vectorized step function for faster training. The baseline algorithms are based on tabular Q Learning. We use Algorithm \ref{alg:ddgc} for DDGC implementation. As described in Figure \ref{figure:mdp}, the MDP is continuing with all the terminal states acting like sink states i.e. any action taken from those states leads back to the same state. Since this is a discrete MDP, the learnt policies are actually deterministic. The policy mixture comprises of multiple deterministic policies with possibly different weights.

\section{Experimental Details (MuJoCo)}\label{app:jaxgcrl_experiments}

We describe the experimental setup used for experiments based on Brax environments. We split the details into three parts: environments, algorithms and baselines, and network architectures.

\subsection{Environments}
Brax \citep{brax2021github} provides GPU-accelerated, vectorized physics environments. We adapt Brax for the multi-goal setting: the agent receives a sparse binary reward based solely on the visited state, with no goal information in the observation vector. This makes the problem suitable for empirically validating algorithms that must discover and cover a fixed set of goals without being told which goal to pursue. We run $\mathbf{64}$ parallel environment instances simultaneously to accelerate data collection. All experiments use an episode length of $\mathbf{500}$ steps.

Our adaptations introduce the following modifications to the standard Brax environments:
\begin{itemize}
    \item A fixed, finite set of goal states is defined per environment. The observation vector contains no information about the goal set — all goals are present simultaneously, and the agent must discover and cover them purely from the sparse reward signal.
    \item Each environment is continuing (non-terminating, but allows resets) with sparse binary rewards that depend only on the current state. The extrinsic reward is $\mathbf{1}$ when the agent enters any goal state and $\mathbf{0}$ otherwise.
    \item Episodes are reset after a fixed horizon of $\mathbf{500}$ steps, or upon environment-specific termination conditions (e.g., unhealthy posture for Ant).
\end{itemize}

Table~\ref{tab:jaxgcrl_details} describes the state-action space, goal criteria, number of goals, and episode length for each environment.

\begin{table}[ht]
    \centering
    \small
    \begin{tabular}{c|c|c|c|c|c}
        \toprule
        Environment & State Space & Action Space & \# Goals & Goal Criterion & Episode Length \\
        \midrule
        Reacher     & $\sR^{6}$  & $(-1,1)^{2}$  & 10 & Fingertip-goal distance $< 0.01$ & 500 \\
        Pusher      & $\sR^{20}$ & $(-1,1)^{7}$  & 5  & Object-goal distance $< 0.03$    & 500 \\
        Ant         & $\sR^{27}$ & $(-1,1)^{8}$  & 10 & Torso xy-goal distance $< 0.2$   & 500 \\
        Half Cheetah & $\sR^{17}$ & $(-1,1)^{6}$ & 6  & Torso x-goal distance $< 0.15$   & 500 \\
        \bottomrule
    \end{tabular}
    \caption{Description of environments. Goals are fixed across resets. The discount factor is $\gamma = 0.99$ for all environments. Sparse binary rewards are used throughout: the extrinsic reward is $\mathbf{1}$ when the agent visits any goal state and $\mathbf{0}$ otherwise. The Reacher observation consists of joint angles (cos/sin) and fingertip velocity; Pusher of joint positions, velocities, arm tip and object positions; Ant of full proprioceptive state; Half Cheetah of joint positions and velocities. Goal positions for each environment are fixed, listed in the source code, and visualized in Figure~\ref{sec:rl_benchmarks}.}
    \label{tab:jaxgcrl_details}
\end{table}

\subsection{Algorithm \& Baselines}

We compare five algorithms: SAC \citep{haarnoja2018soft} as the pure RL baseline, PC (SAC with SimHash pseudo-count intrinsic reward) \citep{bellemare2016unifying} as a count-based exploration baseline, SMM \citep{lee2019efficient} as a state marginal matching baseline, GFlowNet \citep{bengio2021flow} as a flow-matching baseline, and our proposed algorithm DDGC. All algorithms use the Adam optimizer \citep{kingma2014adam}. The target entropy for SAC-based methods is set to $-|\mathcal{A}|/2$ where $|\mathcal{A}|$ is the action dimension. The soft update (Polyak) coefficient is $\tau = 0.005$ for all SAC-based methods.

The total environment step budgets differ per environment to account for task difficulty: Reacher uses $1 \times 10^6$ steps, Pusher uses $2 \times 10^6$ steps, and Ant and HalfCheetah use $5 \times 10^6$ steps. All algorithms use the same step budget for a given environment.

Hyperparameter selection was performed via grid search over 7 random seeds. Table~\ref{tab:algorithm_details} lists the hyperparameters swept and the values considered.

\begin{table}[!ht]
    \small
    \centering
    \begin{tabular}{c|p{7cm}}
        \toprule
        Algorithm & Hyperparameters Swept \\
        \midrule
        SAC
        &
        \texttt{actor\_lr} $\in \{10^{-4},\, 3\times10^{-4},\, 10^{-3}\}$\newline
        \texttt{batch\_size} $\in \{512,\, 1024,\, 2048\}$\newline
        \texttt{hidden\_dim} $\in \{256,\, 512\}$\newline
        \texttt{discount} $\in \{0.99,\, 0.995\}$\newline
        \texttt{reward\_scaling} $\in \{1.0,\, 5.0,\, 10.0\}$\newline
        (\texttt{critic\_lr} = \texttt{alpha\_lr} = \texttt{actor\_lr})
        \\
        \midrule
        PC (SAC)
        &
        \texttt{actor\_lr} $\in \{10^{-4},\, 3\times10^{-4},\, 10^{-3}\}$\newline
        \texttt{batch\_size} $\in \{512,\, 1024,\, 2048\}$\newline
        \texttt{hidden\_dim} $\in \{256,\, 512\}$\newline
        \texttt{discount} $\in \{0.99,\, 0.995\}$\newline
        \texttt{reward\_scaling} $\in \{1.0,\, 5.0,\, 10.0\}$\newline
        \texttt{pc\_beta} $\in \{0.01,\, 0.1,\, 1.0\}$\newline
        (\texttt{critic\_lr} = \texttt{alpha\_lr} = \texttt{actor\_lr})
        \\
        \midrule
        SMM (SAC)
        &
        \texttt{actor\_lr} $\in \{10^{-4},\, 3\times10^{-4},\, 10^{-3}\}$\newline
        \texttt{vae\_lr} $\in \{10^{-4},\, 10^{-3}\}$\newline
        \texttt{batch\_size} $\in \{512,\, 1024\}$\newline
        \texttt{hidden\_dim} $\in \{256,\, 512\}$\newline
        \texttt{num\_iterations} $\in \{10,\, 20\}$\newline
        (\texttt{critic\_lr} = \texttt{alpha\_lr} = \texttt{actor\_lr})
        \\
        \midrule
        DDGC
        &
        \texttt{actor\_lr} $\in \{10^{-4},\, 3\times10^{-4},\, 10^{-3}\}$\newline
        \texttt{rnd\_lr} $\in \{10^{-4},\, 10^{-3}\}$\newline
        \texttt{num\_policies} $\in \{5,\, 10\}$\newline
        \texttt{exploration\_fraction} $\in \{0.1,\, 0.2,\, 0.3\}$\newline
        \texttt{hidden\_dim} $\in \{256,\, 512\}$\newline
        \texttt{fitted\_ac\_iters} $\in \{50,\, 100\}$\newline
        (\texttt{critic\_lr} = \texttt{actor\_lr})
        \\
        \midrule
        GFlowNet
        &
        \texttt{actor\_lr} $\in \{10^{-4},\, 3\times10^{-4},\, 10^{-3}\}$\newline
        \texttt{logz\_lr} $\in \{10^{-3},\, 10^{-2}\}$\newline
        \texttt{batch\_size} $\in \{512,\, 1024\}$\newline
        \texttt{hidden\_dim} $\in \{256,\, 512\}$\newline
        \texttt{tb\_lambda} $\in \{0.5,\, 0.9,\, 1.0\}$\newline
        \texttt{reward\_scale} $\in \{1.0,\, 5.0,\, 10.0\}$\newline
        (\texttt{critic\_lr} = \texttt{actor\_lr})
        \\
        \bottomrule
    \end{tabular}
    \caption{Hyperparameters swept for each algorithm. For SAC-based methods, critic and alpha learning rates are set equal to the actor learning rate. The replay buffer capacity is $10^6$ for all SAC-based methods. The minimum replay size before training begins is $10^4$ for SAC and PC, and $5 \times 10^3$ for SMM. For DDGC, the FAC batch size is fixed at 512 and the buffer capacity is 5000 trajectories. The RND network has hidden dimension 256 and output dimension 64.}
    \label{tab:algorithm_details}
\end{table}

\paragraph{PC (SAC) implementation details.}
The pseudo-count baseline augments the SAC reward with a SimHash-based exploration bonus \citep{tang2017exploration}. A random projection matrix $A \in \mathbb{R}^{d_h \times d_s}$ (with $d_h = 16$) maps observations to binary hash codes; a count table of size $10^5$ tracks visit frequencies. The intrinsic bonus for state $s$ is $\beta / \sqrt{n(s)}$ where $n(s)$ is the hash-bucket visit count and $\beta$ is the \texttt{pc\_beta} hyperparameter. The augmented reward $r_{\text{aug}} = r_{\text{ext}} + \beta / \sqrt{n(s')}$ is stored in the replay buffer.

\paragraph{SMM implementation details.}
SMM \citep{lee2019efficient} runs $K$ iterations of fictitious play. In each iteration: (1) a VAE density model $q_\phi(s)$ is fitted on all replay buffer observations for 1000 gradient steps; (2) a fresh SAC policy is trained for $T/K$ steps using the intrinsic reward $r(s) = -\log q_\phi(s)$ (approximated via VAE reconstruction error); (3) transitions are added to a shared replay buffer (never cleared). The final policy is a uniform mixture over all $K$ saved policy checkpoints.

\paragraph{GFlowNet implementation details.}
The GFlowNet baseline adapts the trajectory balance (TB) objective \citep{malkin2022trajectory} to continuous action spaces. A learned scalar $\log Z$ (the log partition function) is maintained alongside the actor. The actor loss combines a critic-based policy gradient term with a TB regularizer:
\[
\mathcal{L} = -\mathbb{E}[\min(Q_1, Q_2)(s, a)] + \lambda_{\text{TB}} \cdot \mathbb{E}\left[(\log Z + \log P_F(a|s) - \log(\beta \cdot r + \epsilon))^2\right]
\]
where $P_F$ is the forward policy (actor), $r$ is the extrinsic reward, $\beta$ is \texttt{reward\_scale}, and $\lambda_{\text{TB}}$ is \texttt{tb\_lambda}. The critic is trained with the standard double-Q Bellman loss. $\log Z$ is updated jointly with the actor via its own Adam optimizer with learning rate \texttt{logz\_lr}. The replay buffer capacity is $10^6$ and training begins after $10^4$ warmup steps.

\paragraph{DDGC implementation details.}
DDGC proceeds in two phases. In the \emph{exploration phase}, a Random Network Distillation (RND) \citep{burda2018exploration} module drives random exploration for a fraction \texttt{exploration\_fraction} of the total step budget. Goal-visiting trajectories discovered during exploration are stored in a goal buffer (capacity: 5000 trajectories). In the \emph{policy mixture phase}, DDGC iteratively constructs a mixture of $K$ policies via Fitted Actor-Critic (FAC). At iteration $k$: (1) trajectories are collected using the most recent policy; (2) the discounted state visitation distribution $d$ is updated incrementally as $d \leftarrow \frac{k-1}{k+1} d_{\text{prev}} + \frac{2}{k+1} d_{\text{new}}$; (3) a modified dataset is constructed where goal transitions receive reward $1 - d(\text{goal})$ (incentivizing under-visited goals) and non-goal transitions receive reward $0$; (4) FAC is run on this dataset for \texttt{fitted\_ac\_iters} iterations; (5) the new policy is added to the mixture with weight $\frac{2}{k+1}$, and existing weights are decayed by $\frac{k-1}{k+1}$.

\subsection{Network Architectures}

All networks are implemented in JAX \citep{jax2018github} using Flax \citep{flax2020github}.

\textbf{Policy (Actor) Network} comprises 2 hidden layers of dimension \texttt{hidden\_dim} (default: 256) with \textit{tanh} activations, followed by two separate linear output heads for the mean $\mu$ and log standard deviation $\log\sigma$ of the action distribution. The log standard deviation is clipped to $[-5, 2]$ for numerical stability. Actions are sampled from a Tanh-squashed Gaussian: $a = \tanh(\mu + \sigma \odot \epsilon)$, $\epsilon \sim \mathcal{N}(0, I)$. The log-probability accounts for the Tanh squashing correction: $\log \pi(a|s) = \log \mathcal{N}(\cdot) - \sum_i \log(1 - a_i^2 + \epsilon)$.

\textbf{Critic (Double Q) Network} uses a double-critic architecture \citep{fujimoto2018addressing} with two independent Q-networks, each comprising 2 hidden layers of dimension \texttt{hidden\_dim} with \textit{tanh} activations followed by a linear output layer. The observation and action vectors are concatenated before being passed to each Q-network. SAC updates use the minimum of the two Q-values for the target.

\textbf{VAE Density Network} (used in SMM) is a Variational Autoencoder \citep{kingma2013auto}. The encoder comprises 2 ReLU-activated hidden layers of dimension 256, followed by separate linear heads for the latent mean and log-variance (latent dimension: 32). The decoder mirrors the encoder with 2 ReLU-activated hidden layers of dimension 256 followed by a linear reconstruction head. The VAE is trained by maximizing the ELBO: $\mathcal{L} = \mathbb{E}[\|s - \hat{s}\|^2] + \text{KL}(\mathcal{N}(\mu, \sigma^2) \| \mathcal{N}(0, I))$.

\textbf{GFlowNet Actor Network} shares the same architecture as the Policy Network above (2 tanh-activated hidden layers, separate mean and log-std heads, log-std clipped to $[-5, 2]$, Tanh-squashed Gaussian actions), but is trained with the trajectory balance objective rather than SAC. A separate learned scalar $\log Z$ (initialized to 0) is optimized jointly with the actor.

\textbf{RND Network} (used in DDGC) consists of a fixed randomly-initialized \emph{target} network and a learned \emph{predictor} network, each with 2 ReLU-activated hidden layers of dimension 256 and a linear output of dimension 64. The RND exploration bonus for state $s$ is the mean squared error between the target and predictor outputs: $b(s) = \|\phi_{\text{target}}(s) - \phi_{\text{pred}}(s)\|^2$. The predictor is trained online to minimize this error, so novel states (with high prediction error) receive high bonuses.

\subsection{Computational Requirements}\label{app:compreq}
All of the experiments were run on a system with standard 16-core CPU, 256GB of memory with one V100 GPU.

Table~\ref{tab:runtimes} reports the average wall-clock runtime for a single training run, averaged across all environments and hyperparameter settings. Notably, our proposed method (DDGC) is computationally efficient, requiring the least amount of wall-clock time among all compared algorithms, while GFlowNets require significantly more time due to the computational overhead of the trajectory balance objective and forward-looking policy requirements.

\begin{table}[h]
    \centering
    \small
    \begin{tabular}{l|c}
        \toprule
        Algorithm & Average Runtime (h:m:s) \\
        \midrule
        GFlowNet & 7:07:16 \\
        SAC      & 2:55:49 \\
        SMM      & 2:56:12 \\
        PC       & 2:25:25 \\
        \textbf{DDGC (Ours)} & \textbf{2:10:31} \\
        \bottomrule
    \end{tabular}
    \caption{Average wall-clock runtime per seed. Times are averaged across all hyperparameters and environments on a single V100 GPU.}
    \label{tab:runtimes}
\end{table}

\vfill

\end{document}